\def\year{2022}\relax
\documentclass[letterpaper]{article} 
\usepackage{aaai22}  
\usepackage{times}  
\usepackage{helvet}  
\usepackage{courier}  
\usepackage[hyphens]{url}  
\usepackage{graphicx} 
\usepackage{natbib}  
\usepackage{caption} 
\DeclareCaptionStyle{ruled}{labelfont=normalfont,labelsep=colon,strut=off} 
\usepackage{algorithm}
\usepackage{algorithmic}

\usepackage{newfloat}
\usepackage{listings}
\floatstyle{ruled}
\newfloat{listing}{tb}{lst}{}
\floatname{listing}{Listing}
\usepackage{amsmath,amssymb,amsthm,amsfonts}
\usepackage{graphicx}
\usepackage{subfigure}
\usepackage{booktabs}

\def\cA{{\cal A}}

\def\cX{{\cal X}}

\newcommand{\bz}{{\bf z}}

\newcommand{\by}{{\bf y}}

\newcommand{\bX}{{\bf X}}
\newcommand{\bx}{{\bf x}}

\newcommand{\mbR}{\mathbb{R}}
\newcommand{\mbE}{\mathbb{E}}

\newcommand{\bc}{\begin{center}}
\newcommand{\ec}{\end{center}}
\newcommand{\be}{\begin{equation}}
\newcommand{\ee}{\end{equation}}
\newcommand{\ba}{\begin{array}}
\newcommand{\ea}{\end{array}}
\newcommand{\bean}{\begin{eqnarray*}}
\newcommand{\eean}{\end{eqnarray*}}
\newcommand{\bea}{\begin{eqnarray}}
\newcommand{\eea}{\end{eqnarray}}
\newcommand{\ben}{\begin{enumerate}}
\newcommand{\een}{\end{enumerate}}
\newcommand{\bed}{\begin{itemize}}
\newcommand{\eed}{\end{itemize}}
\newcommand{\bs}{\begin{slide}}
\newcommand{\es}{\end{slide}}

\newtheorem{theorem}{Theorem}

\newtheorem{definition}{Definition}

\title{EXoN: EXplainable encoder Network}
\author{
    SeungHwan An\textsuperscript{\rm 1},
    Hosik Choi\textsuperscript{\rm 2},
    Jong-June Jeon\textsuperscript{\rm 1}
}
\affiliations{
    \textsuperscript{\rm 1}Dept. of Statistics,
    University of Seoul,
    S. Korea \\
    \textsuperscript{\rm 2}Graduate School,
    Dept. of Urban Big Data Convergence,
    University of Seoul,
    S. Korea \\
    \{dkstmdghks79, choi.hosik, jj.jeon\}@uos.ac.kr,
}

\begin{document}

\maketitle

\begin{abstract}
We propose a new semi-supervised learning method of Variational AutoEncoder (VAE) which yields a customized and explainable latent space by EXplainable encoder Network (EXoN). Customization means a manual design of latent space layout for specific labeled data. To improve the performance of our VAE in a classification task without the loss of performance as a generative model, we employ a new semi-supervised classification method called SCI (Soft-label Consistency Interpolation). The classification loss and the Kullback-Leibler divergence play a crucial role in constructing explainable latent space. The variability of generated samples from our proposed model depends on a specific subspace, called activated latent subspace. Our numerical results with MNIST and CIFAR-10 datasets show that EXoN produces an explainable latent space and reduces the cost of investigating representation patterns on the latent space. 
\end{abstract}

\section{Introduction}
\label{sec:1}

Variational AutoEncoder (VAE) \cite{kingma2013auto, rezende2014stochastic} aims to reconstruct a well-representing latent space and recover an original observation well. In general, the probabilistic model used in VAE is parameterized by neural networks defined as two maps from the domain of observations to the latent space and vice versa. However, since the probability model of the VAE is not written in a closed form, the maximum likelihood method is unsuitable for application to the VAE. As an alternative, the Variational Bayesian method is popularly applied to the model to maximize the Evidence of Lower Bound (ELBO) \cite{jordan1999introduction}.

Plenty of semi-supervised learning methods for the VAE \cite{kingma2014semi, maaloe2016auxiliary, siddharth2017learning, maaloe2017semi, li2019disentangled, feng2021shot, hajimiri2021semi} have been introduced, and especially \cite{maaloe2017semi, hajimiri2021semi} applied the mixture prior distribution such that the VAE model provides explainable latent space according to labels. However, existing semi-supervised VAE models still have practical limitations. 

\cite{kingma2014semi, maaloe2016auxiliary, siddharth2017learning, li2019disentangled, feng2021shot} introduced an additional latent space representing labels and assumed the probabilistic independence of the label and the other latent variables. While the model is simply formulated, the trained latent space cannot provide explainable and measurable quantitative information used to generate a new image by interpolating between two images. So, it is difficult to impose structural restrictions on the latent space, such as the proximity of latent features across some labeled observations. Also, the trained latent space differs according to the training process, so the latent space is not explained consistently even with the same dataset \cite{maaloe2017semi, hajimiri2021semi}. For example, it is difficult to obtain information about interpolated images from latent space before the trained latent space is entirely investigated by observations. Thus, a manual design of the latent space is required for an explainable representation model. 

The current study proposes a new semi-supervised VAE model, wherein we have incorporated the manual design to improve the clarity of the model. The model employs a Gaussian mixture for prior and posterior latent distributions \cite{dilokthanakul2016deep, zheng2019disentangling, willetts2019disentangling, mathieu2019disentangling, guo2020variational} and focuses on constructing an explainable encoder of VAE. 
In our model, the latent space is customized for each component's correspondence to a specific label and proximity. Since the latent space is shared with the decoder and classifier, each centroid of the mixture distribution is trained as an identifier representing a specific label. 
Therefore, the latent space is manually decomposed by labels, and a user can obtain a customized explainable latent space. 


In addition, we propose a measure of the representation power of latent variables with which the importance of latent subspace can be investigated. It is found that the encoder of our VAE model selectively activates only a part of the latent space, and the latent subspace represents the characteristics of generated samples. Because the activation is measured by a posterior variance linked with the encoder, we call the encoder EXoN (EXplainable encoder Network) by borrowing the term in the field of gene biology.

This paper is organized as follows. Section 2 briefly introduces Variational AutoEncoder, and Section 3 proposes our VAE model, including its derivation. Section 4 shows the results of numerical simulations. Concluding remarks follow in Section 5.

\section{Preliminary}
\label{sec:2}

\subsection{VAE for Unsupervised Learning}
\label{sec:2.1}

Let $\bx \in \mathcal{X} \subset \mathbb{R}^m$ and $\bz \in \mathcal{Z} \subset \mathbb{R}^d$ with $d < m$ be the observed data and the latent variable. $p(\bx)$ and $p(\bz)$ denote the probability density functions (pdf) of $\bx$ and $\bz$. Let $p(\bx|\bz)$ be the conditional pdf of $\bx$ for a given $\bz$ and $\bx|\bz \sim N(D(\bz), \beta \cdot {\bf I}),$ where $D: \mathcal{Z} \mapsto \mathcal{X}$, ${\bf I}$ is the $m \times m$ identity matrix and the observation noise $\beta > 0$. $D(\bz)$ is a neural network with parameter $\theta$. To emphasize dependence on parameters, denote $D(\bz)$ and $p(\bx|\bz)$ by $D(\bz;\theta)$ and $p(\bx|\bz; \theta, \beta)$, respectively. $p(\bx|\bz; \theta, \beta)$ is referred to as the decoder of VAE. The term of decoder is also used as the map $\bz \mapsto (D(\bz;\theta), \beta)$, because $p(\bx|\bz; \theta, \beta)$ is fully parameterized by $(D(\bz;\theta), \beta)$.

The parameter of VAE $(\theta, \beta)$ is estimated by Variational method \cite{kingma2013auto, rezende2014stochastic} that maximizes ELBO. The ELBO of $\log p(\bx; \theta, \beta)$ is obtained from the inequality
\bea \label{eq:elbo}
&& \log p(\bx;\theta, \beta) \\
&\geq& \mathbb{E}_{q(\bz|\bx)} [\log p(\bx|\bz;\theta, \beta)]- \mathcal{KL}\left( q(\bz|\bx) \| p(\bz) \right), \nonumber
\eea
where $\mathcal{KL}(q\|p)$ denotes Kullback-Leibler divergence from $p$ to $q$. \cite{kingma2013auto, rezende2014stochastic} employ a neural network with parameter $\phi$ as $q(\bz|\bx;\phi)$, and obtain the objective function under finite samples $x_1, \cdots, x_n$,
\bea \label{eq:elbo1}
\sum_{i=1}^n \mathbb{E}_{q} [\log p(x_i|\bz;\theta, \beta)] - \mathcal{KL}\left( q(\bz|x_i;\phi) \| p(\bz) \right).
\eea
Here, $q(\bz|\bx;\phi)$ is referred to as the encoder of VAE or the posterior distribution over the latent variables. Multivariate Gaussian distribution is a popular choice for $q(\bz|\bx;\phi)$ in which the mean and covariance are given by neural network models. Thus, the parameters of VAE consist of $\phi$ in the encoder and $(\theta, \beta)$ in the decoder. 

In practice, the VAE is fitted by maximizing a stochastically approximated ELBO for \eqref{eq:elbo1} with respect to $(\theta, \beta, \phi)$. The approximated ELBO is given by  
\bea \label{eq:elbo3}
\sum_{i=1}^n -\frac{1}{2} \| x_i - D(z_i;\theta) \|^2 - \beta \cdot \mathcal{KL}(q(\bz|x_i;\phi)\|p(\bz)) \nonumber \\
- \frac{nm}{2} \log 2\pi\beta 
\eea
where $z_i$ is a sample from $q(\bz|x_i;\phi)$ and $\|\cdot\|$ is $l_2$-norm. The VAE is fitted by maximizing \eqref{eq:elbo3} with respect to $(\theta, \beta, \phi)$ \cite{lucas2019don}. The first term in \eqref{eq:elbo3} is the precision of generated samples. Since the KL-divergence is always non-negative, \cite{higgins2016beta, mathieu2019disentangling, li2019disentangled} explain $\beta$ as a tuning parameter regularizing $\theta$ and $\phi$, and the last term of \eqref{eq:elbo3} is considered as constant.

\section{Proposed Model}
\label{sec:3}

We propose a new VAE model with a customized explainable latent space where a conceptual center of a latent distribution for a specific label can be freely assigned. Let $\by \in \mathcal{Y} = \{1, \cdots, K\}$ be a discrete random variable and denote the joint probability density of $(\bx, \by)$ by $p(\bx, \by)$. In this paper, high dimensional observations $\bx$ are partly labeled by $\by$. Denote the sets of indices for labeled and unlabeled samples by $I_L$ and $I_U$. 

\subsection{Model Assumptions}
\label{sec:3.1}

The prior latent distribution for each label is assumed as $\bz|\by=k \sim N(m_k, diag\{s_k^2\})$ and $p(\by = k) = w_k$ for $k \in \mathcal{Y}$ where $diag\{a\}, a \in \mbR^d$ indicates $d \times d$ matrix whose diagonal elements are $a$. The prior distribution marginalized by $\by$ is 
\bean \label{eq:prior1}
p(\bz) = \sum_{k=1}^K w_k \cdot \mathcal{N}(\bz | m_k, diag\{s_k^2\}).
\eean
Here, $m_k \in \mathbb{R}^d$ and $s_k^2 \in \mathbb{R}_+^d$ are pre-determined parameters denoting the conceptual center and dispersion of the latent variable for each label. The choice of $m_k$ and $s_k^2$ is the customization of the latent space. Since $m_k$ and $w_k$ for all $k$ are fixed, the notation of $m_k$ and $diag\{s_k^2\}$ is omitted in $p(\bz)$.

The encoder $p(\bz|\bx)$ is assumed as a mixture distribution, $p(\bz|\bx) = \sum_{k=1}^K p(\by=k|\bx) \cdot p(\bz|\by=k,\bx)$. Since using $p(\bz|\bx)$ is computationally prohibitive, the proposal distribution $w(\by=k|\bx;\eta)$ for $p(\by=k|\bx)$ and $g(\bz|\by=k,\bx;\xi) = \mathcal{N}(\bz|\mu_k(\bx;\xi), diag\{\sigma^2_k(\bx;\xi)\})$ for $p(\bz|\by=k,\bx)$ are introduced. The two proposal distributions are multinomial and Gaussian, whose parameters are modeled by neural networks. The posterior distribution is approximated by 
\bea \label{eq:encoder1}
q(\bz|\bx;\eta,\xi) = \sum_{k=1}^K w(\by=k|\bx;\eta) \cdot g(\bz|\by=k,\bx;\xi),
\eea
where the parameters of the neural networks are $\eta$ and $\xi$. Note that the posterior mixture weight $w(\by=k|\bx;\eta)$ decompose the latent space according to labels. Given image $\bx$ and label $\by$, the posterior mixture weight assigns a latent variable to the mixture component corresponding to the given label $\by$ and the latent space is separated by labels.

The decoder is assumed by 
\bean
p(\bx|\bz; \theta,\beta) = \mathcal{N}(\bx | D(\bz;\theta), \beta \cdot {\bf I}),
\eean
where the true label $\by$ is not included. Even though a natural decoder can be considered as a mixture Gaussian $p(\bx|\bz;\theta,\beta) = \sum_{k=1}^K p(\by=k|\bz) \cdot \mathcal{N}(\bx|D_k(\bz;\theta), \beta \cdot {\bf I})$, the mixture distribution is approximated by a uni-modal Gaussian distribution especially when $p(\by = k|\bz)$ is assumed to be close to 1 for some $k$. The latent space is separated by $w(\by|\bx;\eta)$ in our encoder and we can obtain such $p(\by|\bz)$. Thus, the validity of the assumption mainly depends on the classification performance of $w(\by|\bx;\eta)$. Although the decoder variance $\beta$ is trainable, we fixed it due to computational issues. 


\subsection{EXoN: Semi-Supervised VAE}
\label{sec:3.2}

Our proposed VAE model is derived from the joint pdf $p(\bx, \by)$. $p(\bx, \by)$ is decomposed into $p(\bx)$ and $p(\by|\bx)$, and the derivation of the ELBO only for $\log p(\bx)$ leads to \eqref{eq:objective0}. 
In the derivation, $p(\by|\bx)$ is approximated by the proposal classification model $w(\by|\bx;\eta)$. 
\bea \label{eq:objective0}
&& \log p(\bx,\by;\theta, \beta) \nonumber \\
&=& \log p(\bx;\theta, \beta) + \log p(\by|\bx;\theta, \beta) \nonumber \\
&\geq& \mathbb{E}_{q} [\log p(\bx|\bz; \theta, \beta)] - \mathcal{KL}\left( q(\bz|\bx; \eta, \xi) \| p(\bz) \right) \nonumber \\
&& + \log p(\by|\bx; \theta, \beta) \nonumber \\
&\simeq& \mathbb{E}_{q} [\log p(\bx|\bz; \theta, \beta)] - \mathcal{KL}\left( q(\bz|\bx; \eta,\xi) \| p(\bz) \right) \nonumber \\
&& + \log w(\by|\bx; \eta). 
\eea
The first two terms in \eqref{eq:objective0} are the typical ELBO used in the unsupervised VAE learning, and the remained term is the classification loss. Note that these terms are coupled with the shared parameter $\eta$. The classification loss term plays the role of training the posterior mixture weights. Therefore, $w(\by|\bx;\eta)$ with lower classification error can separate the latent space more clearly by $q(\bz|\bx;\eta,\xi)$.

Whenever $p(\by|\bx; \theta, \beta) = w(\by|\bx; \eta)$,
\bean
\log p(\bx,\by;\theta, \beta) \geq \mathcal{L}(\bx;\theta,\beta,\xi,\eta) + \log w(\by|\bx;\eta),
\eean
where 
\bean
\mathcal{L}(\bx;\theta,\beta,\xi,\eta) &=& \mathbb{E}_{q} [\log p(\bx|\bz;\theta, \beta)] \\
&& - \mathcal{KL}^u\left( q(\bz|\bx;\eta, \xi) \| p(\bz) \right)
\eean
and $\mathcal{KL}^u\left( q(\bz|\bx;\eta, \xi) \| p(\bz) \right)$ is the upper bound of $\mathcal{KL}\left( q(\bz|\bx;\eta, \xi) \| p(\bz) \right)$ \cite{wang2019topic, guo2020variational} which is written in closed-form as 
\bean
&& \mathcal{KL}(w(\by|\bx;\eta) \| p(\by)) \\
&+& \sum_{k=1}^K w(\by=k|\bx;\eta) \\
&& \times \mathcal{KL} \left( q(\bz|\bx,\by=k;\eta, \xi) \| p(\bz|\by=k) \right).
\eean
Therefore, the lower bound on the joint likelihood for the entire dataset is 
\bea \label{eq:objective2}
\sum_{i \in I_L \cup I_U} \mathcal{L}(x_i;\theta,\beta,\xi,\eta) + \sum_{i \in I_L} \log w(y_i|x_i;\eta).
\eea
\cite{kingma2014semi} first employed the classification loss (the second term in \eqref{eq:objective2}) as a penalty function of VAE and \cite{maaloe2016auxiliary, li2019disentangled, siddharth2017learning} use the same penalty function in the subsequent papers. In our semi-supervised VAE model, the penalty function is applied with a similar idea. However, the derivation of the regularized objective function stems from \eqref{eq:objective0} unlike the existing studies. 

\subsection{SCI: Soft-label Consistency Interpolation}
\label{sec:3.3}

The derivation of \eqref{eq:objective2} is mathematically plausible, but it does not guarantee state-of-the-art semi-supervised classification performance.
To improve the performance of the classification task, we propose a new loss function for the pseudo-labeling semi-supervised classification method \cite{iscen2019label, berthelot2019mixmatch, arazo2020pseudo} called SCI (Soft-label Consistency Interpolation) loss. 

The SCI loss consists of three parts, 1) an interpolated new image with a pair of unlabeled images, 2) a pair of pseudo-label of the unlabeled images, and 3) a convex combination of the cross entropy. Let $(\bx^1, \bx^2)$ be a pair of images and let $\Tilde{\bx} = \rho \cdot \bx^1 + (1-\rho) \cdot \bx^2$ be the interpolated image. Denote the discrete probability $w(\by|\bx; \hat{\eta})$ by $f(\bx; \hat{\eta})$. The pseudo-labels of $\bx^1$ and $\bx^2$ are defined by $\Tilde{\by}^1 = f(\bx_1; \hat{\eta})$ and $\Tilde{\by}^2 = f(\bx_2; \hat{\eta})$ for a given estimate $\hat{\eta}$ (true label is used for labeled dataset). Note that the pseudo-label is not a function of $\hat{\eta}$ but a non-trainable quantity determined by $\hat{\eta}$. Then the SCI loss for $\eta$ with $(\bx^1, \bx^2)$ is defined by
\bea \label{eq:sci}
&& \mbox{SCI}((\bx^1, \Tilde{\by}^1), (\bx^2, \Tilde{\by}^2); \eta) \\
&=& \rho \cdot H(\Tilde{\by}^1, f(\Tilde{\bx}, \eta)) + (1-\rho) \cdot H(\Tilde{\by}^2, f(\Tilde{\bx}, \eta)), \nonumber
\eea
where $H(p,q)$ is the cross entropy of the distribution $q$ relative to a distribution $p$. 


The SCI loss is motivated by the assumption of consistency interpolation. 
\begin{definition} \label{def:CI} 
Given two data point $\bx^1, \bx^2$, if
\bean
&& f(\rho \cdot \bx^1 + (1 - \rho) \cdot \bx^2 ; \eta^*) \\
&=& \rho \cdot f(\bx^1; \eta^*) + (1 - \rho) \cdot f(\bx^2; \eta^*)
\eean
for $\forall \rho \in [0, 1]$, then the consistency interpolation \cite{zhang2017mixup, verma2019interpolation} is satisfied for the parameter $\eta^*$.
\end{definition}
The consistency interpolation assumes the existence of a linear map $f$ from an image to a pseudo-label. It is well known that such a mix-up strategy can improve the generalization error \cite{zhang2017mixup}. 

Interestingly, the estimation of $f(\cdot; \eta^*)$ in our VAE model is also used in other existing semi-supervised learning methods \cite{feng2021shot, arazo2020pseudo} and the following algorithm provides a general framework to estimate such a $f(\cdot; \eta^*)$ in the VAE model. 
Let $\eta^{(t)}$ be an estimate of $\eta$ obtained by the $t$th step while training the VAE, and $\eta^{(t+1)}$ be the solution of following optimal interpolation problem \cite{feng2021shot}
\bea \label{eq:optimal_interpolation_problem1}
\min_{\eta} && \rho \cdot \mathcal{KL}\left( f(\bx^1;\eta^{(t)}) \| f(\Tilde{\bx} ; \eta) \right) \nonumber \\
&+& (1 - \rho) \cdot \mathcal{KL}\left( f(\bx^2;\eta^{(t)}) \| f(\Tilde{\bx} ; \eta) \right).
\eea
Then, it is easily shown that
\bea \label{eq:approximated_CI}
&& f(\rho \cdot \bx^1 + (1 - \rho) \cdot \bx^2 ; \eta^{(t+1)}) \nonumber \\ 
&=& \rho \cdot f(\bx^1 ; \eta^{(t)}) + (1 - \rho) \cdot f(\bx^2 ; \eta^{(t)}).
\eea
If there exists $\eta^*$ such that $\eta^{(t)} \rightarrow \eta^*$ as $t \rightarrow \infty$ given two data point $\bx^1, \bx^2$, then \eqref{eq:approximated_CI} finally implies the consistency interpolation. Since \eqref{eq:optimal_interpolation_problem1} is equivalent with \eqref{eq:sci} up to constant, \eqref{eq:sci} is introduced in our proposed objective function. We also found that using stochastically augmented images for SCI loss helps our classifier achieve higher accuracy (see Appendix \ref{app:4} for detailed augmentation techniques).

Finally, the objective function is given by
\bea \label{eq:objective3}
&-& \sum_{i \in I_L \cup I_U} \mathcal{L}(x_i;\theta,\beta,\xi,\eta) - \sum_{i \in I_L} \log w(y_i|x_i;\eta) \nonumber \\
&-& \frac{\lambda_1}{\beta} \sum_{i \in I_L} \log w(y_i|x_i;\eta) \nonumber\\
&+& \frac{\lambda_1}{\beta} \sum_{i \in I_L} \mbox{SCI}((x^1_i, y^1_i), (x^2_i, y^2_i);\eta) \nonumber \\
&+& \frac{\lambda_2(t)}{\beta} \sum_{i \in I_U} \mbox{SCI}((x^1_i, \Tilde{y}^1_i), (x^2_i, \Tilde{y}^2_i);\eta), 
\eea
where $x_i^1 = x_i$, $x_i^2$ is a randomly chosen sample, and $\lambda_2(t)$ is a ramp-up function. If $\lambda_1$ and $\lambda_2(t)$ are set to be large enough, the mixture components of $q(\bz|\bx;\eta,\xi)$ are shrunk toward those of $p(\bz)$ according to the true matching labels, and the latent space is well separated. The shared parameter $\eta$ leads to this adaption of $q(\bz|\bx;\eta,\xi)$ to $p(\bz)$. 

\subsection{Activated Latent Subspace}
\label{sec:3.4}

In this section, we investigated the meaning of the explainability of the latent space. Denote the random vector associated with the distribution of the $k$th component of the mixture distribution (prior or posterior distribution) by $\bz^k = (\bz_1^k, \cdots, \bz_d^k)$. Then, what is the role of the posterior conditional variance $\mbox{Var}(\bz_j^k|\bx;\xi)$ in the encoding process? 

To answer above question, consider an extreme case where $\mbox{Var}(\bz_j^k|\bx;\xi)$ goes to zero for given $k$. If $\mbox{Var}(\bz_j^k|\bx;\xi) \rightarrow 0$, then $\bz_j^k$ becomes constant. $\bz_j^k$ has a deterministic relationship with given $\bx$, implying that $\bz_j^k$ can explain the given data point. Therefore, the value of $\mbox{Var}(\bz_j^k|\bx;\xi)$ represents the explainability of the latent space, and we call $j$ coordinate is activated.

Interestingly, there is a theorem which shows a connection between our objective function \eqref{eq:objective3} and the posterior conditional variance $\mbox{Var}(\bz_j^k|\bx;\xi)$.
\begin{theorem} \label{thm:1}
Let $p(\bz)$ the mixture prior distribution defined in Section \ref{sec:3.1} and let $q(\bz|\bx;\eta,\xi)$ be \eqref{eq:encoder1}. 
\bean
&& \frac{d}{2} - \sum_{k=1}^K \sum_{j = 1}^d \frac{w_k}{2} \mbE_{\bx|\by=k} \left[ \frac{\mbox{Var}(\bz_j^k|\bx;\xi)}{\mbox{Var}(\bz_j^k)} \right] \\
&\leq& \mbE_{\bx} \Big[ \sum_{k=1}^K w(\by=k|\bx;\eta) \\
&& \times \mathcal{KL} \Big( q(\bz|\bx,\by=k;\eta, \xi) \| p(\bz|\by=k) \Big) \Big] \\
&\leq& \sum_{k=1}^K \sum_{j = 1}^d \frac{w_k}{2} \mbE_{\bx|\by=k} \left[ \frac{\mbox{Var}(\bz_j^k)}{\mbox{Var}(\bz_j^k|\bx;\xi)} \right] - \frac{d}{2}.
\eean
\end{theorem}
\begin{proof}
See Appendix \ref{app:1.2}.
\end{proof}
Theorem \ref{thm:1} means that the KL-divergence upper bound in our objective function is bounded by the ratio of the prior and posterior variances. The lower bound of Theorem \ref{thm:1} can be re-written as $\sum_{k=1}^K \sum_{j = 1}^d w_k/2 \cdot \mbox{Var}_{\bx|\by=k} [\mbE(\bz_j^k|\bx;\xi)] / \mbox{Var}(\bz_j^k)$, so it can be interpreted as the coefficient of determination $R^2$ which measures the proportion of latent variable variation explained by an observation.

On the other hand, a small $\beta$ increases the weight of the generated sample precision term and indirectly decreases the relative weight of the KL-divergence term in \eqref{eq:objective3}. It implies that the KL-divergence term is relaxed (not fully minimized).
So, the relaxation of the KL-divergence term induces an increase in the upper bound of Theorem \ref{thm:1}. It means that there exists an activated coordinate $j$ such that $\mbox{Var}(\bz_j^k|\bx;\xi)$ shrinks to zero for some $k$, as the prior mixture weight $w_k$ and variance $\mbox{Var}(\bz_j^k)$ are fixed due to pre-design. Thus, Theorem \ref{thm:1} shows that the relaxation of the KL-divergence obtained by tuning $\beta$ is closely related to controlling the explainability of the latent space for $\bx$.

Based on Theorem \ref{thm:1}, we propose a statistics V-nat (VAE-natural unit of information) which screens such activated latent coordinates, 
\bean
\log \mbE_{\bx|\by=k} \left[ \mbox{Var}(\bz_j^k) / \mbox{Var}(\bz_j^k|\bx;\xi) \right].
\eean
Additionally, we define the set of activated latent coordinates as activated latent subspace for $k$-labeled dataset for some $\delta > 0$,
\bean
\cA_k(\delta) = \bigg\{ j \in [d]: \log \mbE_{\bx|\by=k} \left[ \frac{\mbox{Var}(\bz_j^k)}{\mbox{Var}(\bz_j^k|\bx;\xi)} \right] > \delta \bigg\}
\eean
where $[d] = \{1, \cdots, d\}$ 
Our numerical study found that the subspace represents the informative characteristics of generated samples, and the subspace can be effectively used to produce a high-quality image (see Section \ref{sec:4.2}). These results are consistent with those of VQ-VAE \cite{van2017neural} that an image generation process mainly depends on a nearly deterministic encoding map, and the refinement of the map can improve the quality of generated images. 

\section{Experiments}
\label{sec:4}


\subsection{MNIST Dataset}
\label{sec:4.1}

We have used the MNIST dataset \cite{lecun2010mnist} to consider the 2-dimensional latent space in our VAE model. The values were scaled in the range of -1 to 1. The encoder returns the 10-component Gaussian mixture distribution parameters, mixing probability, mean vector, and diagonal covariance elements. Thus, the encoder maps $\cX$ to $\left( (0,1) \times \mathbb{R}^{2} \times \mathbb{R}_+^{2} \right)^{10}$. Especially, the mixing probabilities are produced by the classifier in the encoder. Gumbel-Max trick \cite{gumbel1954statistical, jang2016categorical} is used for sampling discrete latent variables. Detailed network architectures of the encoder, decoder and classifier are described in Appendix \ref{app:2} Table \ref{table:mnist_network}, \ref{table:mnist_network_classifier}. 

\begin{figure}[ht]
    \centering 
    \includegraphics[width=0.5\columnwidth]{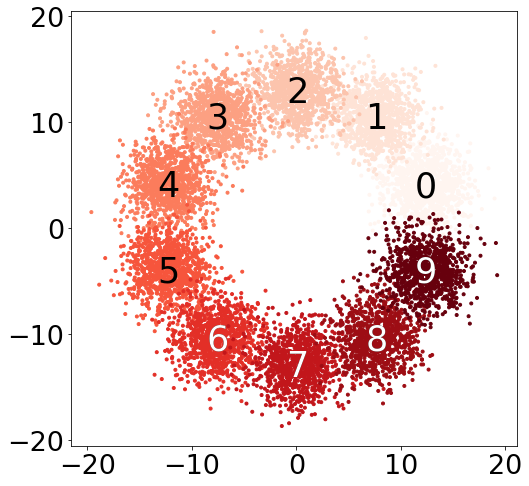}
    \caption{Scatter plot of samples from $p(\bz)$.}
    \label{fig:our_prior}
\end{figure}
The customized $p(\bz)$ is illustrated in Figure \ref{fig:our_prior}. Each label is assigned counterclockwise from 3 O'clock on the component of the mixture. Note that Figure \ref{fig:our_prior} illustrates an example of conceptual centers. The $k$th component of the mixture distribution corresponds to the distribution of the digit $(k-1)$ in the MNIST dataset, for $k=1,\cdots,10$ (see Appendix \ref{app:4} for detailed pre-design settings).

Our model is compared with \cite{kingma2014semi, maaloe2016auxiliary, li2019disentangled, hajimiri2021semi, feng2021shot}, and the simulation results show that our fitted model achieves a competitive classification performance, 3.12\% error with 59,900 unlabeled and 100 labeled images. (see Appendix \ref{app:2} Table \ref{table:mnist_cls_comparison} for comparison results). For implementation details, see Appendix \ref{app:4}. 

\begin{figure*}[ht]
\centering
\includegraphics[width=0.9\linewidth]{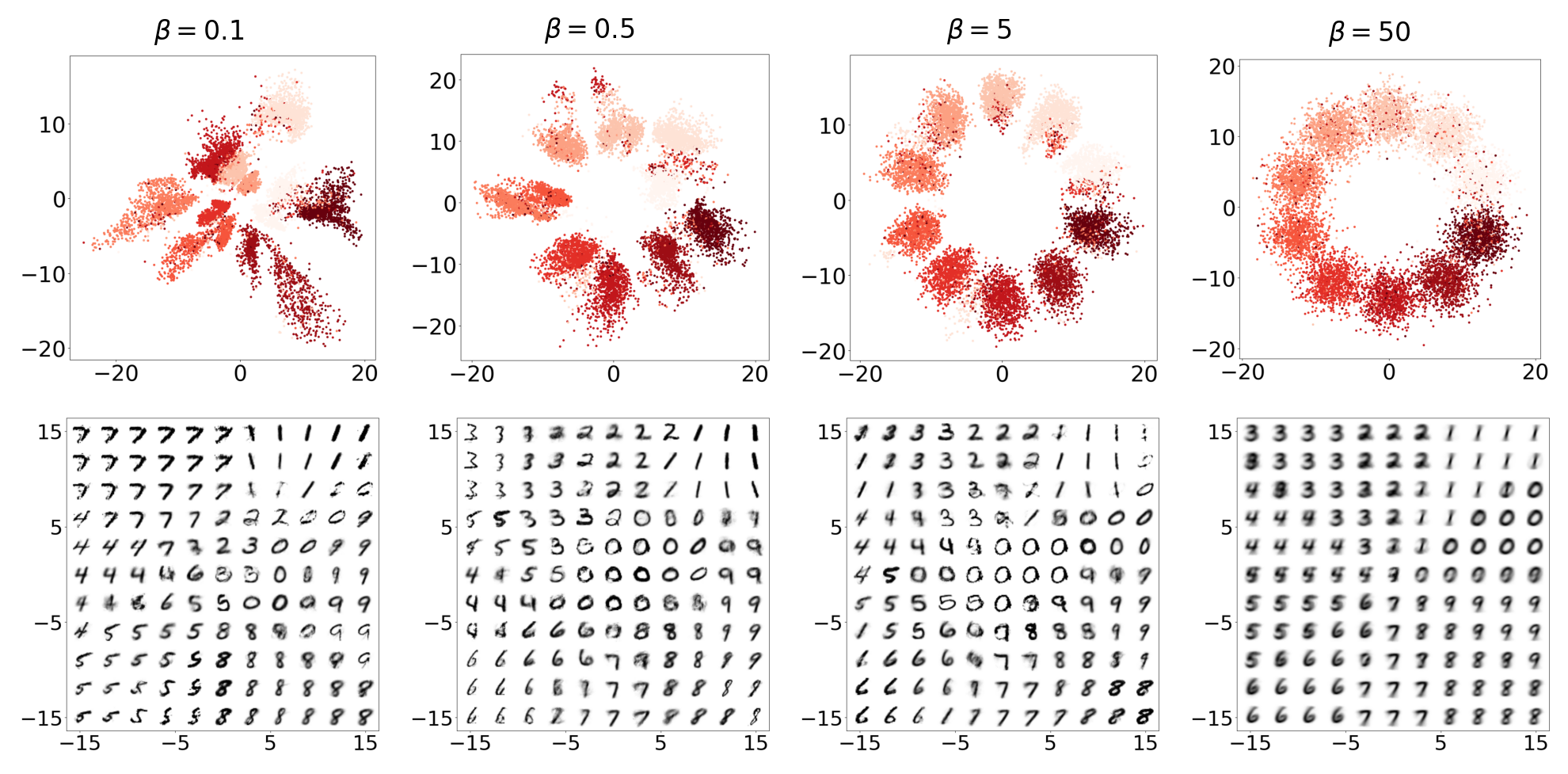}
\caption{Top row: scatter plot of $z \sim q(\bz|x;\eta,\xi)$ given the test dataset. Bottom row: generated images from grid points on the latent space.}
\label{fig:posterior_path}
\end{figure*}

\subsubsection{Effect of Regularizations}
\label{sec:4.1.1}

First, the role of tuning parameter $\beta$ of \eqref{eq:objective3} is investigated in fitted latent space \citep{lucas2019don}. The top panels of Figure \ref{fig:posterior_path} display the samples from $q(\bz|x_i;\eta,\xi)$ where $x_i$ is an observation in the test dataset, and the bottom panels show images generated from grid points on the latent space. 

The top panels demonstrate that the large $\beta$ regularizes $q(\bz|x;\eta,\xi)$ to $p(\bz)$ by indirectly increasing the weight of the KL-divergence term in the objective function. The bottom panels illustrate that each generated image exactly matches the label defined on the pre-designed latent space. Also, it is confirmed that generated images are naturally interpolated on the pre-designed latent space according to our arrangement of conceptual centers. These results show that the proposed VAE yields an explainable latent space with the labels. See Appendix \ref{app:2.1} for additional evaluation results (the negative average single-scale structural similarity (negative SSIM) \cite{wang2004image, zheng2019disentangling}, the classification error, and the KL-divergence) with various $\beta$ values.

\subsubsection{Diversity of Generated Images}
\label{sec:4.1.2}

The bottom panels of Figure \ref{fig:posterior_path} show that the images are generated with lower diversity for larger $\beta$. This result can be explained by the mutual information $I(\cdot,\cdot)$. The expectation of $\mathcal{K}\mathcal{L}^u(q(\bz|\bx;\eta,\xi)\|p(\bz))$ in \eqref{eq:objective3} is rewritten in terms of mutual information as 
\bea \label{eq:info_obj}
&& \mbE_{p(\bx)} \left[\mathcal{K}\mathcal{L}^u(q(\bz|\bx;\eta,\xi)\|p(\bz))\right] \nonumber \\
&=& I(\bx, \by;\eta) + \mathcal{KL}(w(\by;\eta) \| p(\by)) \nonumber \\
&& + \mbE_{p(\by)} [I(\bx, \bz|\by;\eta,\xi)] \nonumber \\
&& + \mbE_{p(\by)} [\mathcal{KL}(q(\bz|\by;\eta,\xi) \| p(\bz|\by))]. 
\eea
See Appendix \ref{app:1.3} for detailed derivation.

To maximize \eqref{eq:objective3} under a large $\beta$, \eqref{eq:info_obj} should be nearly zero and the mutual information between $\bx|\by$ and $\bz|\by$ be close to zero for all classes $\by$ as well. The conditional independence of $\bx$ and $\bz$ implies that $D(\bz;\theta)$ does not depend on $\bz$ for each $\by$, because $\bx|\bz \sim N(D(\bz;\theta), \beta \cdot {\bf I})$ is assumed. For this reason, as shown in the bottom row of Figure \ref{fig:posterior_path}, the latent mixture component for a specific label is not able to capture the complex pattern of observations that belong to a corresponding label when $\beta$ is large.

\subsubsection{Customized Latent Space}
\label{sec:4.1.3}


\begin{figure}
    \centering
    \subfigure[Parted-VAE \cite{hajimiri2021semi}]{\includegraphics[width=0.47\columnwidth]{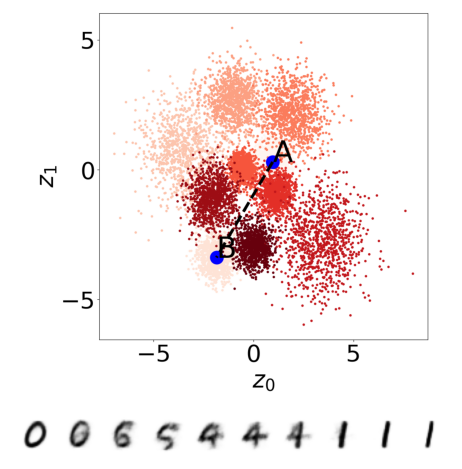}} 
    \subfigure[The EXoN ($\beta=5$)]{\includegraphics[width=0.47\columnwidth]{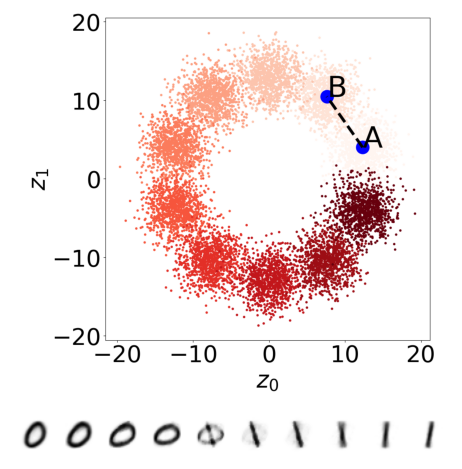}} 
    \caption{The interpolation path from point A (label 0) to B (label 1) and reconstructed images.}
    \label{fig:latent_space_comparison}
\end{figure}

Figure \ref{fig:latent_space_comparison} shows two series of reconstructed images from interpolation on the prior structure. The latent variable A and B are sampled from mixture components of labels 0 and 1, respectively. Suppose the interpolation path from point A to B produces interpolated images with labels 0 and 1. However, the left panel of Figure \ref{fig:latent_space_comparison} shows reconstructed images whose labels are neither 0 nor 1 in the middle of the interpolation path. 

It means that the interpolation path is unpredictable before training with Parted-VAE \cite{hajimiri2021semi}. However, the interpolation path with our model consists of only interpolated images with labels 0 and 1 since our latent space can be manually designed (see the right panel of Figure \ref{fig:latent_space_comparison}). Furthermore, we can pre-determine the interpolation path and the resolution of interpolation by controlling the proximity between mixture components (see Appendix \ref{app:2.2}). 

\subsection{CIFAR-10 Dataset}
\label{sec:4.2}

We apply our VAE model to the CIFAR-10 dataset \cite{krizhevsky2009learning} of which images have 10 labels. We evaluated our model with 400 labeled images per class and 46,000 unlabeled images. All values of the images are scaled to the range of $(-1, 1)$. 256-dimensional latent space and the 10-mixture Gaussian distribution are employed. Each component of the prior mixture distribution has a separate mean vector, and all components share the same covariance (see Appendix \ref{app:4} for detailed pre-designed priors, hyper-parameters, and implementation settings). Gumbel-Max trick \cite{gumbel1954statistical, jang2016categorical} is used for sampling discrete latent variables. The network architectures of the encoder, decoder and classifier are shown in Appendix \ref{app:3} Table \ref{table:cifar10_network}, \ref{table:cifar10_network_classifier}. 

\subsubsection{Comparison}
\label{sec:4.2.2}

\begin{table*}[ht]
  \centering
  \begin{tabular}{lrr}
    \toprule
    models & error (\%) & Inception Score \\
    \midrule
    $\Pi$-model (4.5M) \cite{laine2016temporal} & 17.53 $\pm$ 0.15 & - \\
    *VAT (1.5M) \cite{miyato2018virtual} & 10.55 $\pm$ 0.05 & - \\ 
    MixMatch (5.8M) \cite{berthelot2019mixmatch} & \textbf{5.91} $\pm$ 0.36 & - \\
    PLCB (4.5M) \cite{arazo2020pseudo} & 7.82 $\pm$ 0.18 & - \\
    M2 (4.5M) \cite{kingma2014semi} & 27.30 $\pm$ 0.25 & 1.86 $\pm$ 0.02 \\
    Parted-VAE (5.8M) \cite{hajimiri2021semi} & 34.00 $\pm$ 1.67 & 1.73 $\pm$ 0.20 \\
    SHOT-VAE (5.8M) \cite{feng2021shot} & 6.45 $\pm$ 0.28 & \textbf{3.43} $\pm$ 0.05 \\
    \midrule
    Ours (4.5M) ($\beta=0.01$) & 7.45 $\pm$ 0.48 & 3.14 $\pm$ 0.05 \\
    \bottomrule
  \end{tabular}
\caption{Mean and standard deviation of test dataset classification error for 4,000 labeled dataset and the Inception Score from 5 replicates of the experiment. The numbers in parenthesis are the number of classifier parameters. The Inception Scores are computed using 10,000 generated images. * means that the result is from the original paper.}
\label{table:comparison}
\end{table*}

\begin{figure}[ht]
    \centering
    \subfigure[SHOT-VAE (between same classes)]{
    \includegraphics[width=.9\columnwidth]{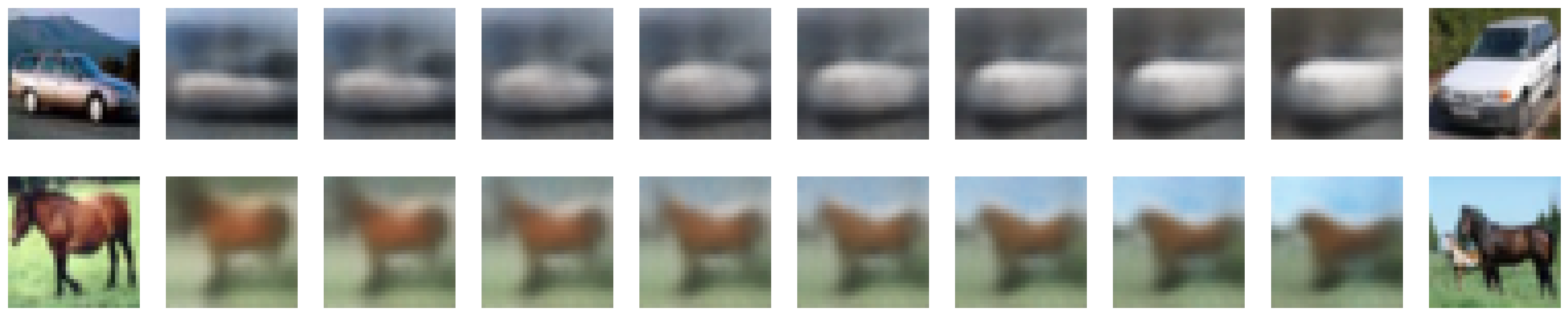}
    }
    \subfigure[EXoN ($\beta=0.05$) (between same classes)]{
    \includegraphics[width=.9\columnwidth]{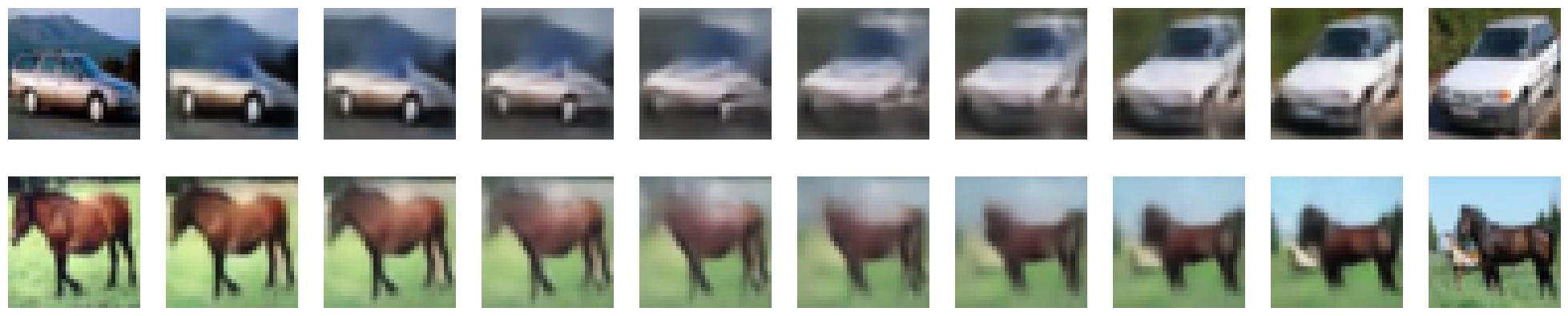}
    }
    \subfigure[Parted-VAE (between different classes)]{
    \includegraphics[width=.9\columnwidth]{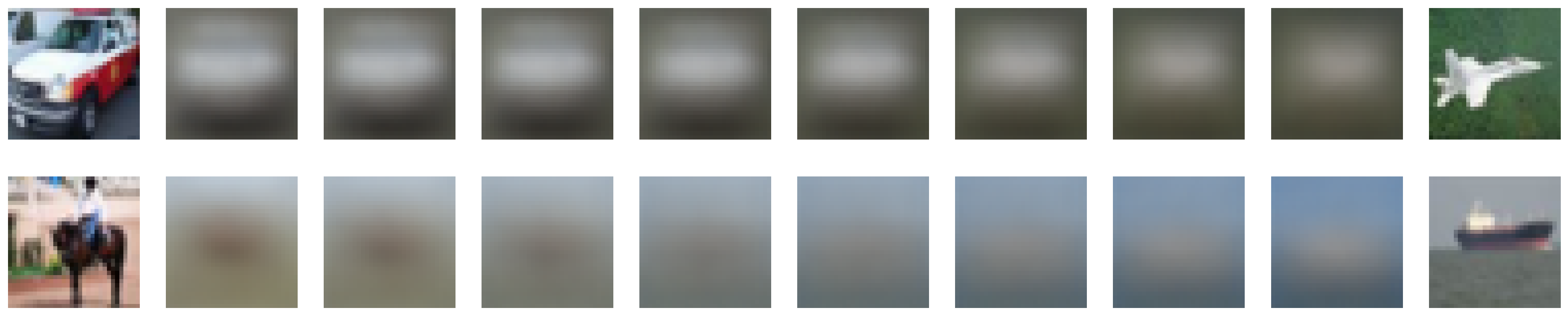}
    }
    \subfigure[EXoN ($\beta=0.05$) (between different classes)]{
    \includegraphics[width=.9\columnwidth]{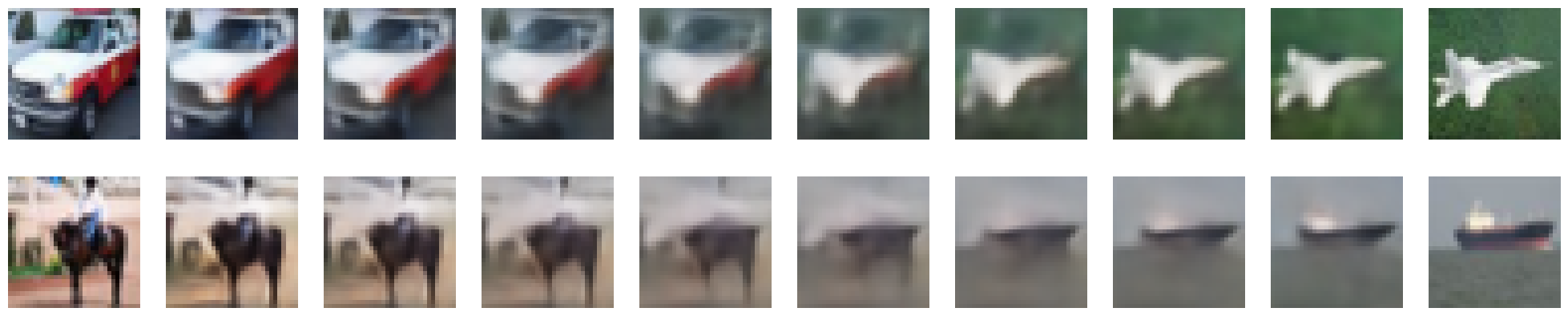}
    }
\caption{
Series of images obtained by interpolating two points on the latent space(or subspace), which various models produce. The most left and right are original images.
}
\label{fig:interpolation}
\end{figure}

Table \ref{table:comparison} shows quantitative comparison results of the classification performance in the CIFAR-10 benchmark setting and image generation performance. The classification models \cite{laine2016temporal, miyato2018virtual, berthelot2019mixmatch, arazo2020pseudo} are not generative, so the Inception Score is not computed. Although the EXoN is not the best in the Inception Score, our model has a remarkable image interpolation quality. It is shown in Figure \ref{fig:interpolation} which visualizes the reconstructed images from interpolation. This Figure indicates that our model achieves much better recovering and interpolation performance than other models \cite{hajimiri2021semi, feng2021shot} (see Appendix Figure \ref{fig:train_recon} for more reconstructed images by our model).


\subsubsection{Activated Latent Subspace}
\label{sec:4.2.1}

\begin{figure}[ht]
    \centering
    \subfigure[EXoN ($\beta=0.05$)]{
    \includegraphics[width=.75\columnwidth]{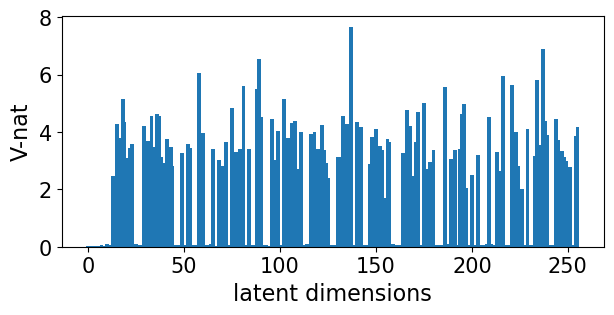}
    \label{fig:measure}
    }
    \subfigure[EXoN ($\beta=0.05$)]{
    \includegraphics[width=.9\columnwidth]{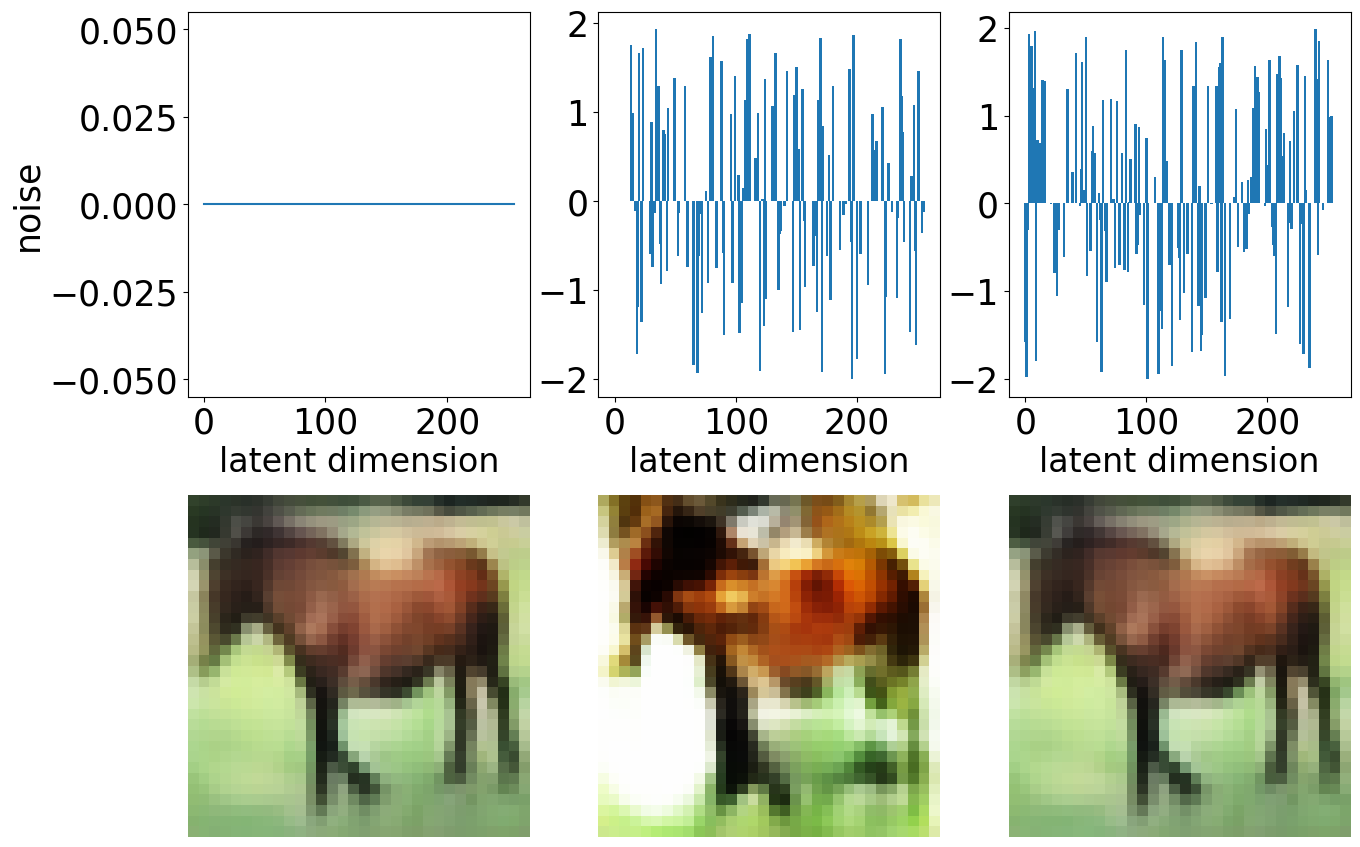}
    \label{fig:signal}
    }
\caption{
(a) Plot of V-nat values for automobile class.
(b) Top row: visualization of noises added to the latent variable, bottom row: generated image from the corresponding perturbed latent variable.
}
\end{figure}

The properties of the activated latent subspace are investigated. For $\beta=0.05$, the latent space with $|\cA_k(1)|=115$ is obtained for the automobile class. The associated V-nat values are displayed in Figure \ref{fig:measure} (see Appendix \ref{app:3.1} for other classes). We generate images on three cases of latent variables, 1) the original latent variable, 2) the latent variable perturbed by the uniform noise $U(-2,2)$ on $\cA_k(1)$, 3) the latent variable perturbed by the uniform noise $U(-2,2)$ on $\cA^c_k(1)$. Figure \ref{fig:signal} displays examples of noise values and generated images. 

The two images of left and right in Figure \ref{fig:signal} seem identical even though their latent variables are far from each other. While, the middle image of the bottom panel is distorted compared with the other images, which confirms that only the activated latent subspace determines the features of generated images. Motivated by these numerical results in Figure \ref{fig:signal}, we manipulated images using only the activated latent subspace (see results in Appendix \ref{app:3.2}). 

\begin{table*}[ht]
    \centering
    \begin{tabular}{lrrrrr}
    \toprule
    $\beta$ & 0.01 & 0.05 & 0.1 & 0.5 & 1  \\
    \midrule
    Inception Score & 3.14 $\pm$ 0.05 & 3.36 $\pm$ 0.05 & 3.37 $\pm$ 0.10 & 3.17 $\pm$ 0.08 & 2.86 $\pm$ 0.06 \\
    $|\cA_k(1)|$ & 189.2 $\pm$ 22.7 & 114.6 $\pm$ 4.22 & 82.8 $\pm$ 12.50 & 36.00 $\pm$ 5.87 & 22.00 $\pm$ 1.10 \\
    error (\%) & 7.45 $\pm$ 0.48 & 7.27 $\pm$ 0.19 & 7.51 $\pm$ 0.25 & 8.18 $\pm$ 0.63 & 9.50 $\pm$ 1.07 \\
    \bottomrule
    \end{tabular}
\caption{For various $\beta$ values, mean and standard deviation of the Inception Score \protect \cite{salimans2016improved}, cardinality of activated latent subspace $|\cA_k(1)|$ of automobile class and test dataset classification error with 4,000 labeled dataset from 5 replicates of the experiment.}
\label{table:path}
\end{table*}
Theorem \ref{thm:1} theoretically supports that relaxation of KL-divergence (shrinkage of $\beta$) leads to a large set $\cA_k(1)$. Meanwhile, \eqref{eq:info_obj} implies that strong regularization of KL-divergence (increasing $\beta$) leads to decreasing $\mbE_{p(\by)} [I(\bx, \bz|\by;\eta,\xi)]$, which results in a small set $\cA_k(1)$. These theoretical results are shown numerically in Table \ref{table:path}. 

\section{Conclusion and Limitations}
\label{sec:5}

This paper proposes a new method to construct an explainable latent space in semi-supervised learning VAE. The explainable latent space can be obtained by combining classification losses (cross-entropy and SCI loss) and the relaxed KL-divergence term in our objective function. By the classification, the latent space can be effectively decomposed according to the mixture components and the true labels of observations. In addition, it is shown that the relaxation of the KL-divergence increases the dimension of the activated latent subspace which determines the characteristics of the images generated from our proposed model. The activated latent subspace can be discovered through the V-nat measure.

In short, the explainability of our proposed model is demonstrated by two points 1) the label-specific conceptual means and variances of the latent distribution, and 2) the features of the given images can be analyzed in association with the activated latent subspace. Thus, the practical advantage of our method is that the manually interpolated images corresponding to the user's specific purpose can be achieved by customizing the prior distribution structure in advance. We guess that the clustering method would replace the role of the classification loss in our VAE model, and leave the development of the EXoN in unsupervised learning as our future work. 

\bibliography{aaai22}

\clearpage
\appendix

\section{Appendix}

\subsection{Theoretical Derivations}
\label{app:1}


\subsubsection{Upper bound of the KL-divergence}
\label{app:1.1}

We use the upper bound of the KL-divergence from $p(\bz)$ to $q(\bz|\bx;\eta, \xi)$ \cite{wang2019topic, guo2020variational}  by
\bea \label{eq:kl_upperbound}
&& \mathcal{KL}^u\left( q(\bz|\bx;\eta, \xi) \| p(\bz) \right) \nonumber \\
&\equiv& \mathcal{KL}(w(\by|\bx;\eta) \| p(\by)) \nonumber \\
&+& \sum_{k=1}^K w(\by=k|\bx;\eta) \nonumber \\
&& \times \mathcal{KL} \Big( \mathcal{N}(\bz | \mu_k(\bx;\xi), diag\{(\sigma_k(\bx;\xi))^2\}) \| \nonumber \\
&& \mathcal{N}(\bz | \mu_k^0, diag\{(\sigma_k^0)^2\}) \Big).
\eea

\subsubsection{Proof of Theorem 1}
\label{app:1.2}

Let $\bz = (\bz_1, \cdots, \bz_d) \sim N(\mu, diag\{(\sigma)^2\})$ and $\bz|\bx \sim N(\mu(\bx), diag\{(\sigma(\bx))^2\})$. For understanding following equations, we denote the $j$th elements of $\mu$, $\mu(\bx)$, $(\sigma)^2$ and $(\sigma(\bx))^2$ by $\mbE (\bz_j)$, $\mbE (\bz_j|\bx)$, $\mbox{Var}(\bz_j)$ and $\mbox{Var}(\bz_j|\bx)$, respectively. The KL-divergence between two univariate normal distributions leads to a close form of the KL-divergence from $N(\mu, diag\{(\sigma)^2\})$ to $N(\mu(\bx), diag\{(\sigma(\bx))^2\})$, 
\bea \label{eq:kl_expectation}
&& \mbE_{\bx} \Big[ \mathcal{KL} \Big( N(\mu(\bx), diag\{(\sigma(\bx))^2\}) \| \nonumber \\ 
&& N(\mu, diag\{(\sigma)^2\}) \Big) \Big] \nonumber \\
&=& \mbE_{\bx} \Big[ \frac{1}{2} \sum_{j=1}^d \Big( \frac{[\mbE(\bz_j|\bx) - \mbE(\bz_j)]^2}{\mbox{Var}(\bz_j)} + \frac{\mbox{Var}(\bz_j|\bx)}{\mbox{Var}(\bz_j)} \nonumber \\
&& + \log \frac{\mbox{Var}(\bz_j)}{\mbox{Var}(\bz_j|\bx)} - 1 \Big) \Big] \nonumber \\
&=& \frac{1}{2} \sum_{j=1}^d \Big( \frac{\mbox{Var}_{\bx} [\mbE(\bz_j|\bx)]}{\mbox{Var}(\bz_j)} + \frac{\mbE_{\bx} [\mbox{Var}(\bz_j|\bx)]}{\mbox{Var}(\bz_j)} \nonumber
\eea
\bea
&& + \mbE_{\bx} \left[ \log \frac{\mbox{Var}(\bz_j)}{\mbox{Var}(\bz_j|\bx)} \right] \Big) - \frac{d}{2} \nonumber \\
&=& \frac{1}{2} \sum_{j=1}^d \left( \frac{\mbox{Var}(\bz_j)}{\mbox{Var}(\bz_j)} + \mbE_{\bx} \left[ \log \frac{\mbox{Var}(\bz_j)}{\mbox{Var}(\bz_j|\bx)} \right] \right) - \frac{d}{2} \nonumber \\
&=& \frac{1}{2} \sum_{j=1}^d \mbE_{\bx} \left[ \log \frac{\mbox{Var}(\bz_j)}{\mbox{Var}(\bz_j|\bx)} \right].
\eea

Using above equality \eqref{eq:kl_expectation}, the expectation of the second term in \eqref{eq:kl_upperbound} can be written as:
\bea
&& \sum_{k=1}^K \mbE_{\bx} \Big[w(\by=k|\bx;\eta) \nonumber \\
&& \times \mathcal{KL} \Big(\mathcal{N}(\bz | \mu_k(\bx;\xi), diag\{(\sigma_k(\bx;\xi))^2\})  \| \nonumber \\
&& \mathcal{N}(\bz | \mu_k^0, diag\{(\sigma_k^0)^2\}) \Big) \Big] \nonumber \\
&=& \sum_{k=1}^K \int_{\bx} p(\bx) w(\by=k|\bx;\eta) \nonumber \\
&& \times \mathcal{KL} \Big(\mathcal{N}(\bz | \mu_k(\bx;\xi), diag\{(\sigma_k(\bx;\xi))^2\})  \| \nonumber \\
&& \mathcal{N}(\bz | \mu_k^0, diag\{(\sigma_k^0)^2\}) \Big) d\bx \nonumber \\
&=& \sum_{k=1}^K \int_{\bx} p(\by=k) p(\bx|\by=k;\eta) \nonumber \\
&& \times \mathcal{KL} \Big(\mathcal{N}(\bz | \mu_k(\bx;\xi), diag\{(\sigma_k(\bx;\xi))^2\}) \| \nonumber \\
&& \mathcal{N}(\bz | \mu_k^0, diag\{(\sigma_k^0)^2\}) \Big) d\bx \nonumber \\
&=& \sum_{k=1}^K w_k^0 \cdot \mbE_{\bx|\by=k} \Big[ \nonumber \\ 
&& \mathcal{KL} \Big(\mathcal{N}(\bz | \mu_k(\bx;\xi), diag\{(\sigma_k(\bx;\xi))^2\}) \| \nonumber \\
&& \mathcal{N}(\bz | \mu_k^0, diag\{(\sigma_k^0)^2\}) \Big) \Big] \nonumber \\
&=& \frac{1}{2} \sum_{k=1}^K w_k^0 \sum_{j=1}^d \mbE_{\bx|\by=k} \left[ \log \frac{(\sigma^0_{kj})^2}{(\sigma_{kj}(\bx;\xi))^2} \right],
\eea
where $\sigma^0_{kj}$ and $\sigma_{kj}(\bx;\xi)$ are the $j$th element of $\sigma_k^0$ and $\sigma_k(\bx;\xi)$, respectively. The second equality above equations  holds for $p(\by|\bx;\eta) = w(\by|\bx;\eta)$.

Let $\bz^k = (\bz_1^k, \cdots, \bz_d^k) \sim N(\mu_k^0, diag\{(\sigma_k^0)^2\})$, the distribution of the $k$th component of the mixture distribution $p(\bz)$. First, the upper bound of \eqref{eq:kl_upperbound} is written as
\bea \label{eq:upper}
&& \mbE_{\bx}\mathcal{KL}^u\left( q(\bz|\bx;\eta, \xi) \| p(\bz) \right) \nonumber \\
&=& \mbE_{\bx}\mathcal{KL}(w(\by|\bx;\eta) \|p(\by)) \nonumber \\
&& + \frac{1}{2} \sum_{k=1}^K w_k^0 \sum_{j = 1}^d \mbE_{\bx|\by=k} \left[ \log \frac{\mbox{Var}(\bz_j^k)}{\mbox{Var}(\bz_j^k|\bx;\xi)} \right] \nonumber \\
&\leq& \mbE_{\bx}\mathcal{KL}(w(\by|\bx;\eta) \|p(\by)) \nonumber \\
&& + \sum_{k=1}^K \sum_{j = 1}^d \frac{w_k^0}{2} \log \mbE_{\bx|\by=k} \left[  \frac{\mbox{Var}(\bz_j^k)}{\mbox{Var}(\bz_j^k|\bx;\xi)} \right] \nonumber \\
&\leq& \mbE_{\bx}\mathcal{KL}(w(\by|\bx;\eta) \|p(\by)) \nonumber \\
&& + \sum_{k=1}^K \sum_{j = 1}^d \frac{w_k^0}{2} \mbE_{\bx|\by=k} \left[  \frac{\mbox{Var}(\bz_j^k)}{\mbox{Var}(\bz_j^k|\bx;\xi)} \right] - \frac{d}{2},
\eea
by Jensen's inequality and $\log x \leq x-1$.

Similarly, the lower bound of \eqref{eq:kl_upperbound} is derived as
\bea \label{eq:lower}
&& \mbE_{\bx}\mathcal{KL}^u\left( q(\bz|\bx;\eta, \xi) \| p(\bz) \right) \nonumber \\
&=& \mbE_{\bx}\mathcal{KL}(w(\by|\bx;\eta) \|p(\by)) \nonumber \\
&& - \frac{1}{2} \sum_{k=1}^K w_k^0 \sum_{j = 1}^d \mbE_{\bx|\by=k} \left[ \log \frac{\mbox{Var}(\bz_j^k|\bx;\xi)}{\mbox{Var}(\bz_j^k)} \right] \nonumber \\
&\geq& \mbE_{\bx}\mathcal{KL}(w(\by|\bx;\eta) \|p(\by)) \nonumber \\
&& - \sum_{k=1}^K \sum_{j = 1}^d \frac{w_k^0}{2} \log \mbE_{\bx|\by=k} \left[ \frac{\mbox{Var}(\bz_j^k|\bx;\xi)}{\mbox{Var}(\bz_j^k)} \right] \nonumber \\
&\geq& \mbE_{\bx}\mathcal{KL}(w(\by|\bx;\eta) \|p(\by)) \nonumber \\ 
&& - \sum_{k=1}^K \sum_{j = 1}^d \frac{w_k^0}{2} \mbE_{\bx|\by=k} \left[ \frac{\mbox{Var}(\bz_j^k|\bx;\xi)}{\mbox{Var}(\bz_j^k)} \right] + \frac{d}{2}.
\eea

Therefore, combining \eqref{eq:upper} and \eqref{eq:lower}, 
\bea
&& \frac{d}{2} - \sum_{k=1}^K \sum_{j = 1}^d \frac{w_k^0}{2} \mbE_{\bx|\by=k} \left[ \frac{\mbox{Var}(\bz_j^k|\bx;\xi)}{\mbox{Var}(\bz_j^k)} \right] \\
&\leq& \mbE_{\bx}\mathcal{KL}^u\left( q(\bz|\bx;\eta, \xi) \| p(\bz) \right) - \mbE_{\bx}\mathcal{KL}(w(\by|\bx;\eta) \|p(\by)) \nonumber \\
&\leq& \sum_{k=1}^K \sum_{j = 1}^d \frac{w_k^0}{2} \mbE_{\bx|\by=k} \left[  \frac{\mbox{Var}(\bz_j^k)}{\mbox{Var}(\bz_j^k|\bx;\xi)} \right] - \frac{d}{2}, 
\eea
completes the proof.



\subsubsection{Mutual Information Derivation}
\label{app:1.3}

Here, we denote the mutual information as $I(\cdot,\cdot)$ and the expectation of the KL-divergence upper bound \eqref{eq:kl_upperbound} is
\bea \label{eq:kl_upper}
&& \mbE_{p(\bx)}\mathcal{KL}^u\left( q(\bz|\bx;\eta, \xi) \| p(\bz) \right) \nonumber \\
&=& \mbE_{p(\bx)} \left[ \mathcal{KL}(w(\by|\bx;\eta) \| p(\by)) \right] \nonumber \\
&& + \mbE_{p(\bx)} \Big[ \sum_{k=1}^K w(\by=k|\bx;\eta) \nonumber \\
&& \times \mathcal{KL}\left( q(\bz|\bx, \by=k;\xi) \| p(\bz|\by=k) \right) \Big].
\eea

We assume that $w(\by|\bx;\eta) = p(\by|\bx;\eta)$.
\bea
&& \mbE_{p(\bx)} [\mathcal{KL}\left( w(\by|\bx;\eta) \| p(\by) \right)] \nonumber \\
&=& \mbE_{p(\bx)} \mbE_{w(\by|\bx;\eta)} \left[\log \frac{w(\by|\bx;\eta)}{w(\by)}\right] \nonumber \\
&& + \mbE_{p(\bx)} \mbE_{w(\by|\bx;\eta)} [\log w(\by;\eta)] \nonumber \\
&& - \mbE_{p(\bx)} \mbE_{w(\by|\bx;\eta)} [\log p(\by)] \nonumber \\
&=& \mbE_{p(\bx)} [\mathcal{KL}(w(\by|\bx;\eta) \| w(\by;\eta))] \nonumber \\ 
&& + \mbE_{w(\by;\eta)} \left[\log \frac{w(\by;\eta)}{p(\by)}\right] \nonumber \\
&=& \int p(\bx) w(\by|\bx;\eta) \log \frac{w(\by|\bx;\eta)}{w(\by;\eta)} d\by d\bx \nonumber \\
&& + \mathcal{KL}(w(\by;\eta) \| p(\by)) \nonumber \\
&=& \int p(\bx, \by;\eta) \log \frac{p(\bx, \by;\eta)}{p(\bx) w(\by;\eta)} d\by d\bx \nonumber \\
&& + \mathcal{KL}(w(\by;\eta) \| p(\by))  \nonumber \\
&=& I(\bx,\by;\eta) + \mathcal{KL}(w(\by;\eta) \| p(\by)),
\eea
where $w(\by;\eta) = \mbE_{p(\bx)}[w(\by|\bx;\eta)]$ and $p(\bx,\by;\eta) = p(\bx) w(\by|\bx;\eta)$.

For the second term in \eqref{eq:kl_upper}, we define $q(\bz|\by;\eta,\xi)$ as
\bea
&& \mbE_{p(\bx|\by;\eta)}[q(\bz|\bx,\by;\xi)] \nonumber \\
&=& \int p(\bx|\by;\eta) q(\bz|\bx,\by;\xi) d\bx \nonumber \\
&=& \int \frac{p(\bx)w(\by|\bx;\eta)}{p(\by)} \frac{q(\bz,\by|\bx;\eta,\xi)}{w(\by|\bx;\eta)} d\bx \nonumber \\
&=& \int \frac{q(\bz,\by,\bx;\eta,\xi)}{p(\by)} d\bx = \frac{q(\bz,\by;\eta,\xi)}{p(\by)} \nonumber \\
&=& q(\bz|\by;\eta,\xi) \nonumber
\eea
where $p(\bx|\by;\eta) = p(\bx)w(\by|\bx;\eta) / p(\by)$.

Therefore, 
\bea
&& \mbE_{p(\bx)} \Big[\sum_{k=1}^K w(\by=k|\bx;\eta) \nonumber \\
&& \times \mathcal{KL}\Big( q(\bz|\bx,\by=k;\xi) \| p(\bz|\by=k) \Big)\Big] \nonumber \\
&=& \mbE_{p(\bx) w(\by|\bx;\eta) q(\bz|\bx,\by;\xi)} \left[ \log \frac{q(\bz|\bx,\by;\xi)}{p(\bz|\by)} \right] \nonumber \\
&=& \mbE_{p(\bx) w(\by|\bx;\eta) q(\bz|\bx,\by;\xi)} \left[ \log \frac{q(\bz|\bx,\by;\xi) q(\bz|\by;\eta,\xi)}{p(\bz|\by) q(\bz|\by;\eta,\xi)} \right] \nonumber \\
&=& \mbE_{p(\bx) w(\by|\bx;\eta) q(\bz|\bx,\by;\xi)} \left[ \log \frac{q(\bz|\bx,\by;\xi)}{q(\bz|\by;\eta,\xi)} \right] \nonumber \\
&& + \mbE_{p(\bx) w(\by|\bx;\eta) q(\bz|\bx,\by;\xi)} \left[ \log \frac{q(\bz|\by;\eta,\xi)}{p(\bz|\by)} \right] \nonumber \\
&=& \mbE_{p(\by) p(\bx|\by;\eta) q(\bz|\bx,\by;\xi)} \left[ \log \frac{q(\bz|\bx,\by;\xi)}{q(\bz|\by;\eta,\xi)} \right] \nonumber \\
&& + \mbE_{p(\by) p(\bx|\by;\eta) q(\bz|\bx,\by;\xi)} \left[ \log \frac{q(\bz|\by;\eta,\xi)}{p(\bz|\by)} \right] \nonumber 
\eea
\bea
&=& \mbE_{p(\by) q(\bx,\bz|\by;\eta,\xi)} \left[ \log \frac{q(\bz|\bx,\by;\eta,\xi)}{q(\bz|\by;\eta,\xi)} \right] \nonumber \\
&& + \mbE_{p(\by) q(\bz|\by;\eta,\xi)} \left[ \log \frac{q(\bz|\by;\eta,\xi)}{p(\bz|\by)} \right] \nonumber \\
&=& \mbE_{p(\by)} \int q(\bx,\bz|\by;\eta,\xi) \log \frac{q(\bx,\bz|\by;\eta,\xi)}{q(\bz|\by;\eta,\xi)p(\bx|\by)} d\bx d\bz \nonumber \\
&& + \mbE_{p(\by)} [\mathcal{KL}(q(\bz|\by;\eta,\xi) \| p(\bz|\by))] \nonumber \\
&=& \mbE_{p(\by)} [I(\bx,\bz|\by;\eta,\xi)] \nonumber \\
&& + \mbE_{p(\by)} [\mathcal{KL}(q(\bz|\by;\eta,\xi) \| p(\bz|\by))],
\eea
where $p(\bx|\by;\eta)p(\by) = p(\bx)w(\by|\bx;\eta)$ and $q(\bx,\bz|\by;\eta,\xi) = p(\bx|\by;\eta) q(\bz|\bx,\by;\xi)$.

In conclusion, the expectation of the KL-divergence upper bound is written in mutual information terms as
\bea
&& \mbE_{p(\bx)}[\mathcal{KL}^u\left( q(\bz|\bx;\eta, \xi) \| p(\bz) \right)] \nonumber \\
&=& I(\bx,\by;\eta) + \mathcal{KL}(w(\by;\eta) \| p(\by)) \nonumber \\
&& + \mbE_{p(\by)} [I(\bx,\bz|\by;\eta,\xi)] \nonumber \\
&& + \mbE_{p(\by)} [\mathcal{KL}(q(\bz|\by;\eta,\xi) \| p(\bz|\by))].
\eea

\subsection{MNIST Dataset}
\label{app:2}

\subsubsection{Evaluation with various tuning parameter $\beta$s}
\label{app:2.1}

\begin{table*}[ht]
    \centering
    \begin{tabular}{lrrrrrrrr}
    \toprule
    $\beta$ & 0.1 & 0.25 & 0.5 & 0.75 & 1 & 5 & 10 & 50 \\
    \midrule
    negative SSIM & 0.43 & 0.438 & 0.44 & 0.445 & 0.435 & 0.439 & 0.418 & 0.316 \\
    KL-divergence & 19.654 & 9.654 & 8.481 & 7.684 & 6.75 & 4.064 & 2.729 & 1.855 \\
    error(\%) & 3.29 & 3.23 & 3.25 & 3.33 & 3.46 & 3.74 & 4.07 & 3.38 \\
    \bottomrule
    \end{tabular}
\caption{The negative SSIM, KL-divergence, and the classification error of 100 labeled dataset for various $\beta$ values.}
\label{table:mnist_path}
\end{table*}

Denote the set of indices for the test dataset $I_{test}$. The KL-divergence measures how different the trained posterior distribution is from the pre-designed prior distribution and is defined by:
\bea \label{eq:kl_metric}
\frac{1}{|I_{test}|} \sum_{i \in I_{test}} \mathcal{KL}^u\left( q(\bz|x_i;\eta, \xi) \| p(\bz) \right).
\eea

The negative average single-scale structural similarity (negative SSIM) of $\bX$ is defined by
\bea \label{eq:ssim}
\frac{1}{2} \bigg( 1 - \frac{1}{|\bX|^2} \sum_{(\bx, \bx') \in \bX \times \bX} \mbox{SSIM}(\bx, \bx') \bigg),
\eea
where $\mbox{SSIM}(\bx, \bx')$ for $\bx, \bx' \in \bX$ is the similarity measure between two images $\bx$, $\bx'$ and has a value on $[-1,1]$ \cite{wang2004image, zheng2019disentangling}. So, negative SSIM has a value on $[0,1]$ and indicates how many diverse images $\bX$ consists of. For $\bX$ being the set of images generated from a fitted VAE model, $\mbox{negative SSIM}(\bX)$ indicates how expressive the model is. We consist $\bX$ with the images generated from our trained decoder, where each image is produced by 31$\times$31 equally spaced grid points on the 2-dimensional latent space. 

The classification error is given by 
\bea \label{eq:classification_error}
\frac{1}{|I_{test}|} \sum_{i \in I_{test}} \mathbb{I} \left(y_i \neq \arg \max_k (w(\by=k|x_i; \eta)) \right),
\eea
where $\mathbb{I}(\cdot)$ is a indicator function. The classification error shows the degree of discrepancy between the assigned label by the posterior probability and the true observation label. Thus, the VAE model with low classification error can separate the latent space into subsets on which the labels of observations are well identified. 

Table \ref{table:mnist_path} indicates that the KL-divergence mostly depends on $\beta$. Because $\beta$ indirectly controls the weight of the KL-divergence loss term, the KL-divergence for a large $\beta$ dominates in our objective function \cite{lucas2019don}. Also, we can observe that the diversity of generated sample (\mbox{negative SSIM}) increases as $\beta$ decreases. For all $\beta$s, the latent space is clearly separated according to the true labels. It means that the classification performance does not depend on $\beta$ since the weight of additional classification loss terms is scaled by $1/\beta$.

\subsubsection{Controlling Proximity and Interpolation}
\label{app:2.2}

\begin{figure}[ht]
    \centering
    \includegraphics[width=0.9\columnwidth]{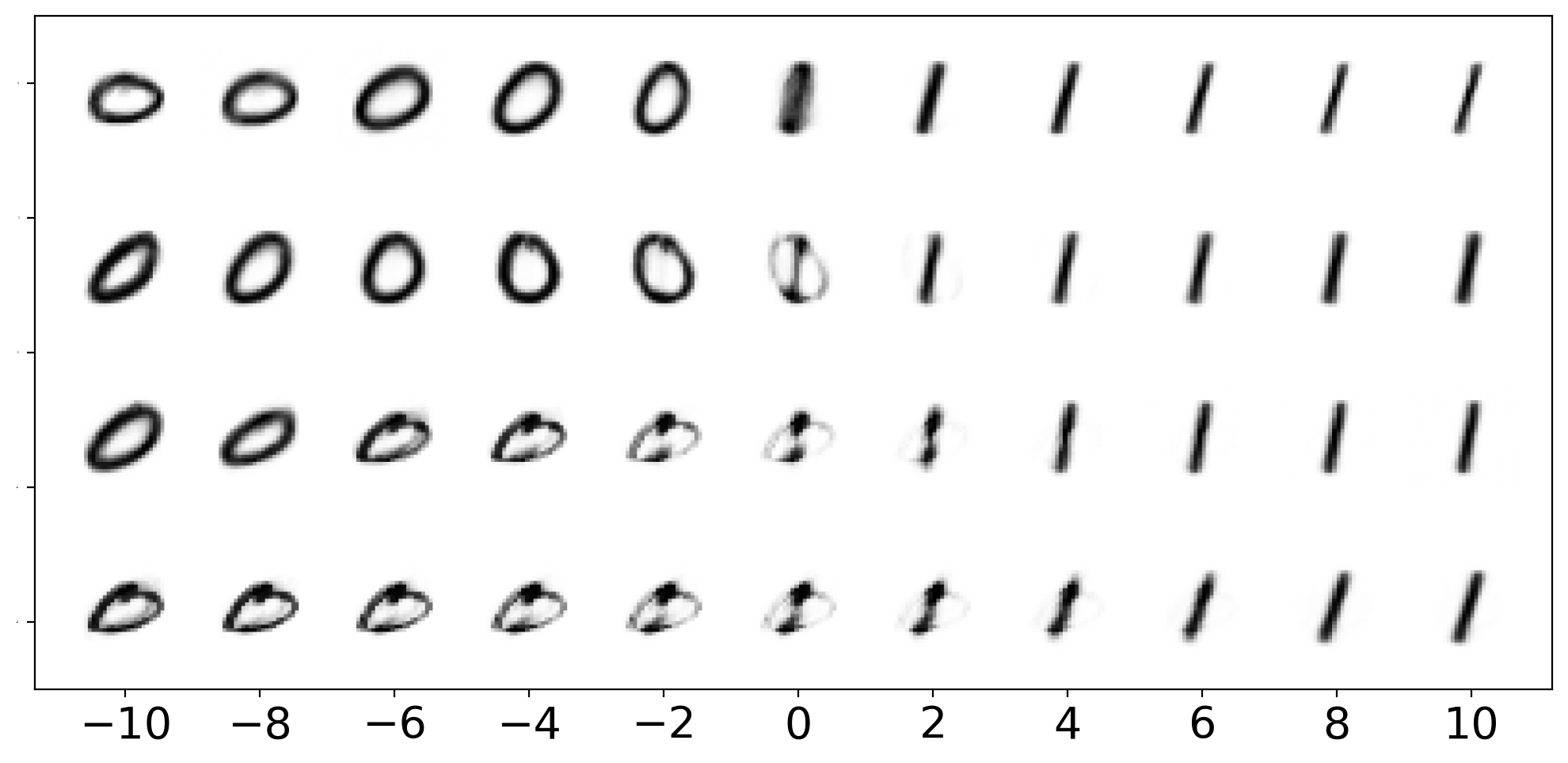}
    \caption{Interpolated images produced by various pre-designs of the prior distribution. From top to bottom, the distance between the center of the two components increases as 8, 16, 24, and 32.}
    \label{fig:radius_diversity}
\end{figure}

We investigate the patterns of interpolated images according to various pre-designs of the proximity of the prior mixture components. We use a subset of the MNIST dataset with only 0 and 1 labels, so the 2-component Gaussian mixture distribution is used. All Gaussian components have diagonal covariance matrices, and their diagonal elements are all 4. We set the location parameters as $(-r, 0), (r, 0)$, which determines the proximity. We set the distances between the two location parameters as 8, 16, 24, and 32. The images were generated from 11 points of equally spacing line segment from $(-10, 0)$ to $(10, 0)$ on the 2-dimensional latent space. All experiment settings are equally given by those of MNIST dataset analysis in Section \ref{sec:4.1}. Figure \ref{fig:radius_diversity} shows that the interpolated images vary more slowly if the two location parameters of $p(\bz)$ are farther from each other and the latent space is effectively adapted to pre-designed characteristics.


\begin{table}[ht]
  \centering
  \begin{tabular}{cc}
    \toprule
    \textbf{encoder} & \textbf{decoder} \\
    \midrule
    (28, 28, 1) image & $d$-dimension latent variable\\
    Flatten & Dense(128), ReLU \\ 
    Dense(256), ReLU & Dense(256), ReLU \\ 
    Dense(128), ReLU & Dense(784), ReLU \\ 
    2 $\times $ K $\times$ Dense($d$), Linear & Reshape(28, 28, 1) \\
    \bottomrule
  \end{tabular}
\caption{Model descriptions of the encoder and decoder. Here, $K$ denotes the number of classes, and $d$ is the dimension of the latent variable.}
\label{table:mnist_network}
\end{table}

\begin{table}[ht]
  \centering
  \begin{tabular}{c}
    \toprule
    \textbf{classifier} \\
    \midrule
    input: 28 $\times$ 28 $\times$ 1 Image \\
    Conv2D(32, 5, 1, `same'), BN, LeakyReLU($\alpha$=0.1) \\
    MaxPool2D(pool size=(2, 2), strides=2) \\
    SpatialDropout2D(rate=0.5) \\
    Conv2D(64, 3, 1, `same'), BN, LeakyReLU($\alpha$=0.1) \\
    MaxPool2D(pool size=(2, 2), strides=2) \\
    SpatialDropout2D(rate=0.5) \\
    Conv2D(128, 3, 1, `same'), BN, LeakyReLU($\alpha$=0.1) \\
    MaxPool2D(pool size=(2, 2), strides=2) \\
    SpatialDropout2D(rate=0.5) \\
    GlobalAveragePooling2D \\
    Dense(64), BN, ReLU \\
    Dense(K), softmax \\
    \bottomrule
  \end{tabular}
\caption{Network structures of the classifier. Here, $K$ denotes the number of classes, and BN is Batch-Normalization layer \cite{ioffe2015batch}. Conv2D hyper-parameters are filters, kernel size, strides, and padding.}
\label{table:mnist_network_classifier}
\end{table}


\begin{table}[ht]
  \centering
  \begin{tabular}{lr}
    \toprule
    models & error(\%) \\
    \midrule
    M2\shortcite{kingma2014semi} & 11.97 \\ 
    M1+M2\shortcite{kingma2014semi} & 3.33 \\
    ADGM\shortcite{maaloe2016auxiliary} &\textbf{0.96} \\
    Disentangled-VAE\shortcite{li2019disentangled} & 2.71 \\
    Parted-VAE\shortcite{hajimiri2021semi} & 9.79 \\
    SHOT-VAE\shortcite{feng2021shot} & 3.12 \\ 
    \midrule
    Ours ($\beta=0.25$) & 3.23 \\
    \bottomrule
  \end{tabular}
\caption{Semi-supervised test classification errors with 100 labeled datasets (except for Parted-VAE, all results are from original papers).}
\label{table:mnist_cls_comparison}
\end{table}

\subsection{CIFAR-10 Dataset}
\label{app:3}

\subsubsection{V-nat of the EXoN}
\label{app:3.1}

\begin{figure*}[ht]
    \centering
    \includegraphics[width=0.6\textwidth]{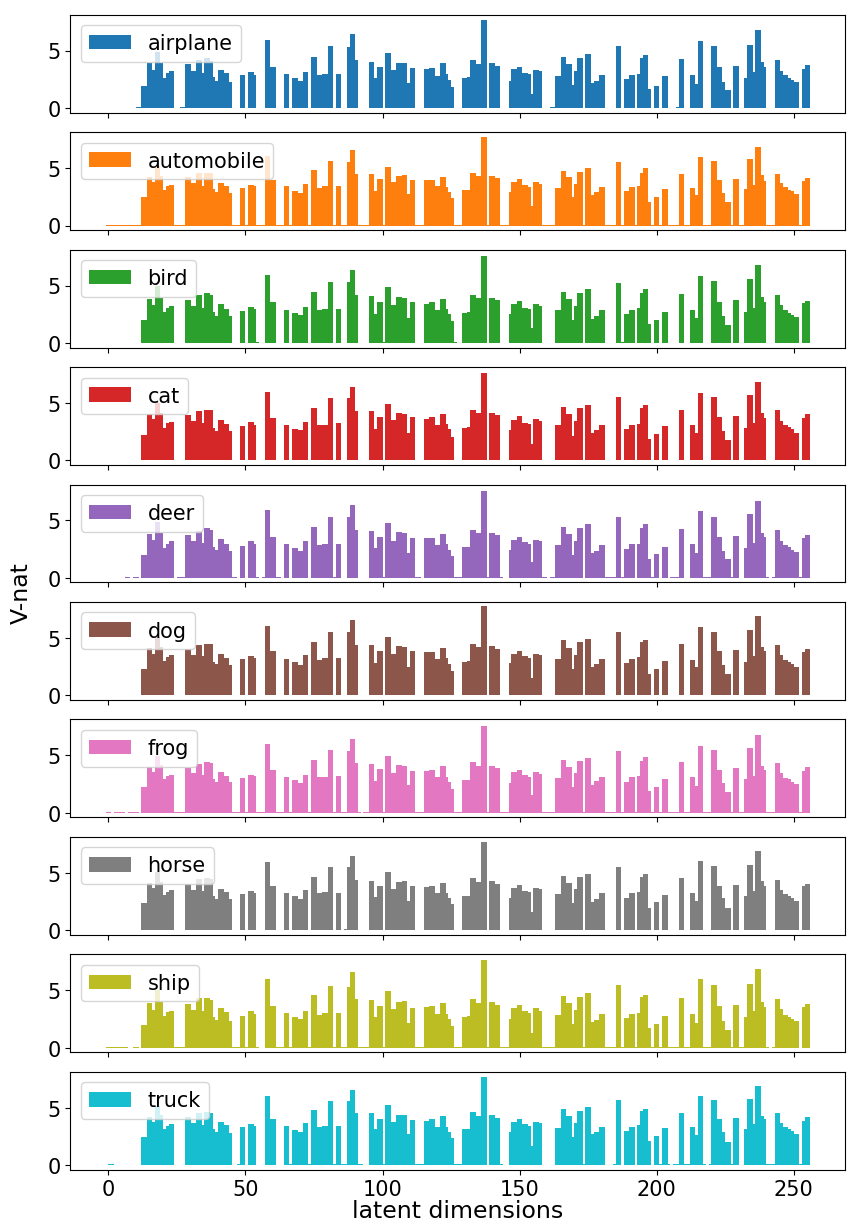}
    \caption{Visualization of 256-dimensional V-nat vector of the EXoN for all classes where $\beta = 0.05$.}
    \label{fig:all_measure}
\end{figure*}

\begin{figure*}[ht]
    \centering
    \includegraphics[width=0.7\textwidth]{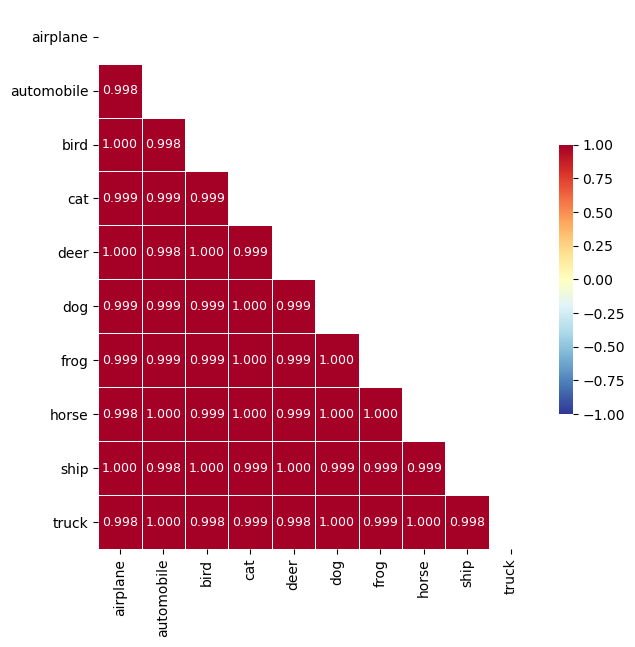}
    \caption{Correlation matrix between V-nat vectors of all classes with $\beta = 0.05$}
    \label{fig:vnat_corr}
\end{figure*}

Figure \ref{fig:all_measure} shows V-nat vectors for all classes in CIFAR-10 datasets, and $j$th element of V-nat vector is $\log \mbE_{\bx|\by=k} \left[ \mbox{Var}(\bz_j^k) / \mbox{Var}(\bz_j^k|\bx;\xi) \right]$. We found that the activated latent subspace of each class (latent dimensions which have a higher V-nat value than $\delta=1$) does not differ significantly from each other.

To show that the variability of generated images depends on the almost same latent subspace for all classes, the correlation matrix between V-nat vectors of all classes with $\beta = 0.05$ is shown in Figure \ref{fig:vnat_corr}. The correlation plot indicates that the V-nat vectors are remarkably correlated (correlation coefficients are almost 1), implying that the diversity of generated images depends on almost the same subspace.

\subsubsection{Manipulation on the activated latent subspace}
\label{app:3.2}

\begin{figure*}[h]
    \centering
    \subfigure[$\beta=0.01$]{
    \includegraphics[width=0.7\textwidth]{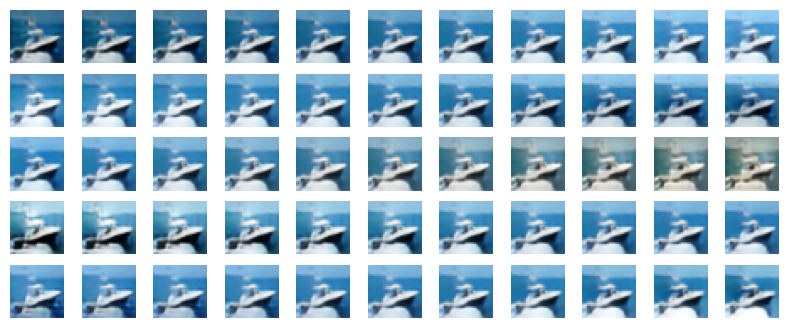}
    }
    \subfigure[$\beta=0.1$]{
    \includegraphics[width=0.7\textwidth]{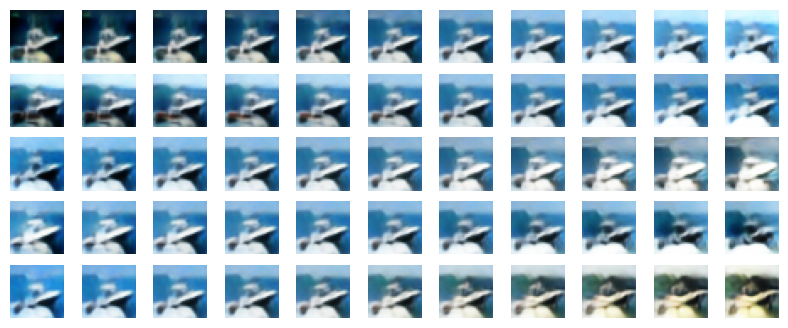}
    }
\caption{Manipulation on the single activated latent subspace and generated images.}
\label{fig:top5_axes}
\end{figure*}

\begin{figure*}[h]
    \centering
    \subfigure[$\beta=0.01$]{
    \includegraphics[width=0.7\textwidth]{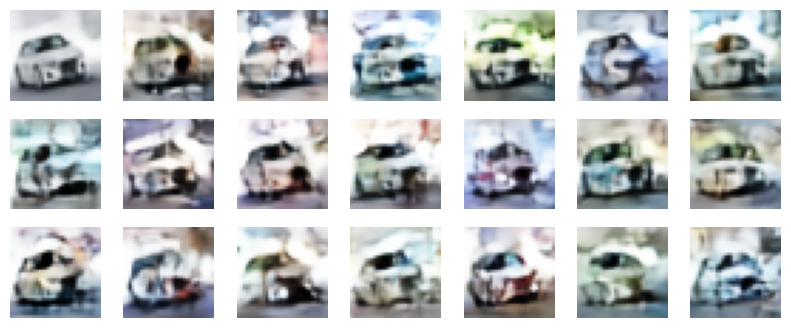}
    }
    \subfigure[$\beta=0.1$]{
    \includegraphics[width=0.7\textwidth]{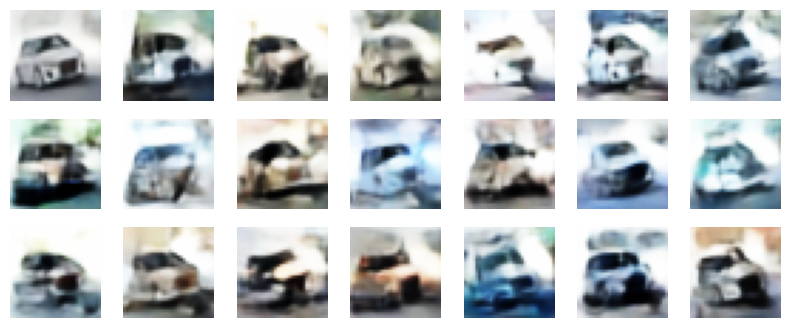}
    }
\caption{Manipulation on all activated latent subspace and generated images.}
\label{fig:blurmany}
\end{figure*}

In Figure \ref{fig:top5_axes}, 5 series of images are generated from latent variables where the value of the activated latent axis is changed from -3 to 3, and values of other latent dimensions are fixed. From top to bottom, latent axes having the top 5 most significant V-nat values are used for each $\beta$. As the value of each activated latent subspace changes, some different features of generated samples are varied (like brightness, the color and the shape of the object, or the color of the background). 

Note that characteristics of images change more significantly where the decoder variance parameter $\beta$ is large because a single activated latent dimension determines relatively more image characteristics due to the small cardinality of the activated latent subspace when $\beta$ is large. Disentangling properties represented by the activated latent subspace is left for our future work.

Furthermore, we manipulated all activated latent subspace, and it is visualized in Figure \ref{fig:blurmany} for each $\beta$. It shows images generated from latent variables perturbed by the uniform noise on $\cA_k(1)$, where the uniform noise is sampled from $U(-1.5, 1.5)$ and automobile class $k$. 

Most tops left images are reconstructed images given the unperturbed latent variable. As in manipulating the single activated latent subspace, this Figure confirms that a single activated latent dimension determines relatively more characteristics of a given image because the cardinality of the activated latent subspace is small for a large $\beta$.


\begin{figure*}[h]
    \centering
    \includegraphics[width=0.7\textwidth]{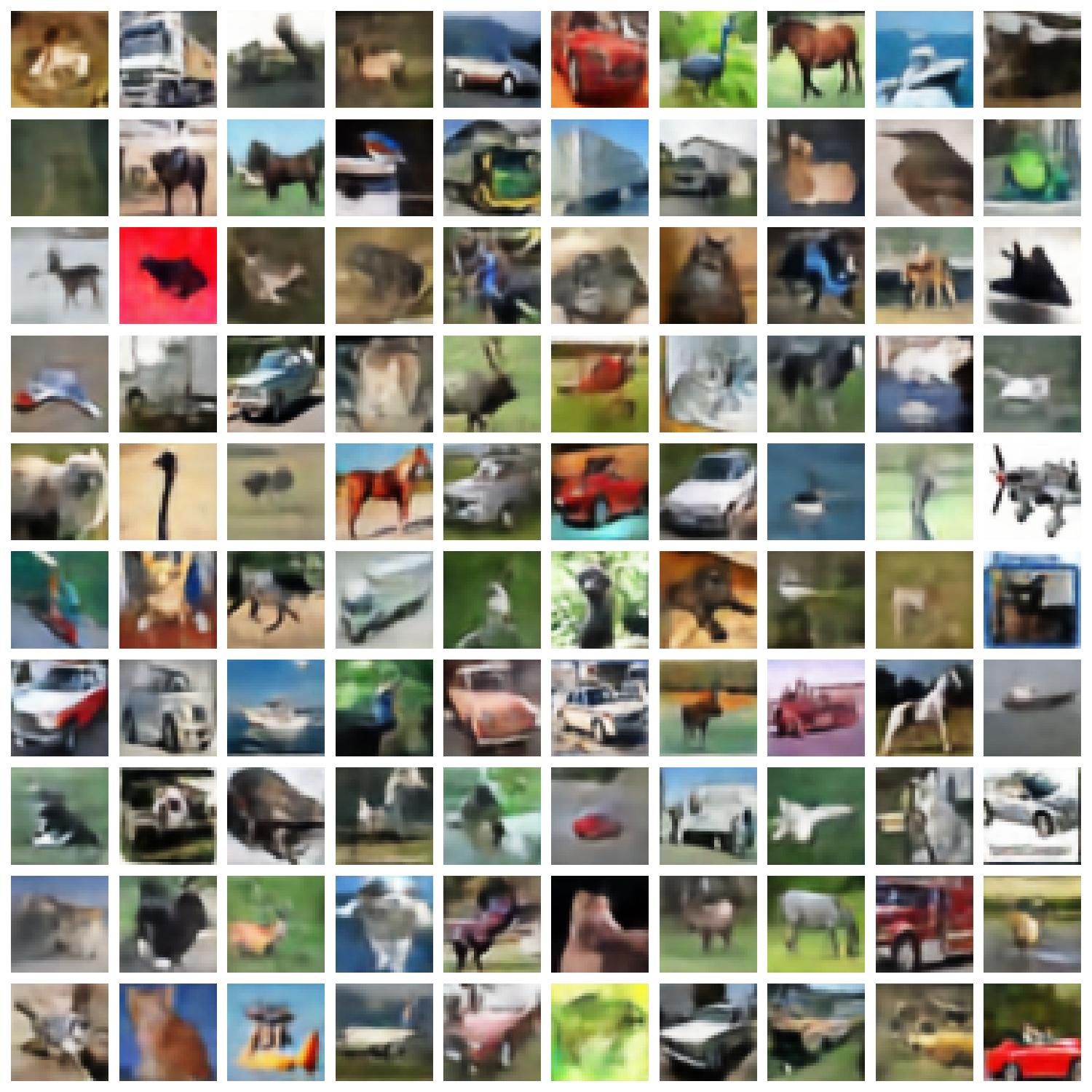}
    \caption{Recovered images given the train dataset where $\beta = 0.05$. It confirms that the proposed model could recover an original observation without loss of information.}
    \label{fig:train_recon}
\end{figure*}

\begin{figure*}[ht]
    \centering
    \includegraphics[width=0.7\textwidth]{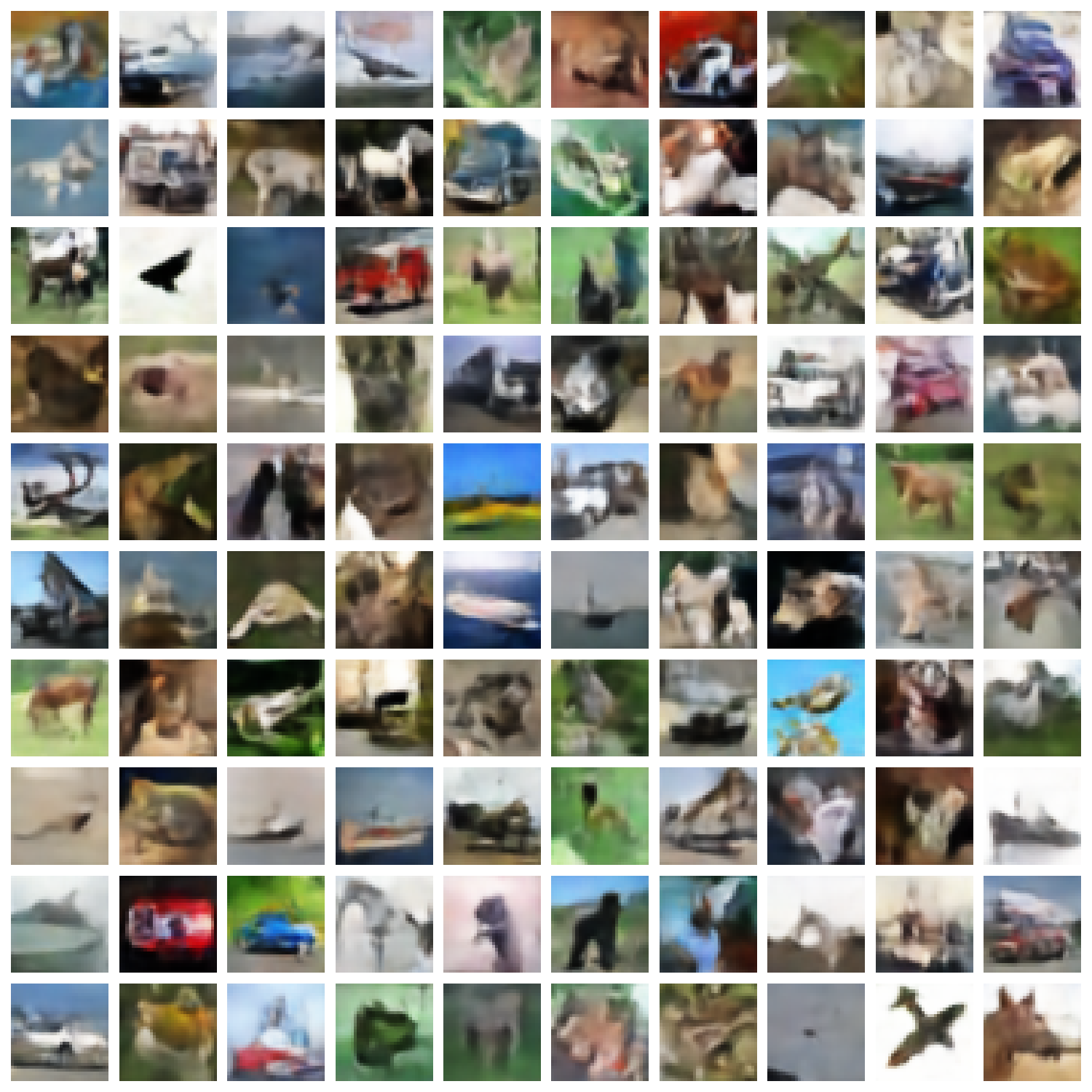}
    \caption{Recovered images given the test dataset where $\beta = 0.05$.}
    \label{fig:test_recon}
\end{figure*}


\subsection{Experiment settings}
\label{app:4}

We run all experiments using Geforce GTX 2080 Ti GPU and 16GB RAM, and our experimental codes are all available with Tensorflow \cite{tensorflow2015-whitepaper}. Pre-designs of the prior distribution:
\bed
    \item MNIST: The conceptual centers and variances are given by $\mu_k^0 = r \cdot \big( \cos(\frac{\pi}{10}(2k-1)), \sin(\frac{\pi}{10}(2k-1)) \big)$, where $r = 4 / \sin(\frac{\pi}{10})$ and $(\sigma_k^0)^2 = (4, 4)$ for $k = 1, \cdots, 10$.
    \item CIFAR-10: The mean vector of $p(\bz)$ consists of two subvectors, a 10-dimensional label-relevant mean which is the one-hot vector for the corresponding label, and a 246-dimensional label-irrelevant mean which has all zero values \cite{zheng2019disentangling}. All covariance matrices are commonly given by $diag\{({0.1}_{10}, {1}_{246})\}$ where ${a}_k$ for ${a} \in \mathbb{R}$ and $k \in \mathbb{N}$ is a $k$-dimensional vector whose elements are all $a$.
\eed

Stochastic image augmentations we used:
\bed
    \item MNIST: random rotation, random cropping
    \item CIFAR-10: random horizontal flip, random cropping
\eed
\begin{table*}[ht]
  \centering
  \begin{tabular}{lrrrrrrr}
    \toprule
    dataset & epochs & batch size($U, L$) & drop rate & $\lambda_1$ & $\lambda_2(e)$ \\
    \midrule
    MNIST & 100 & (128, 32) & 0.5 & 6000 & $\lambda_1 \exp\left(-5 (1 - \min\{1, \frac{c}{10}\})^2\right)$ \\
    CIFAR-10 & 600 & (128, 32) & 0.1 & 5000 & $\lambda_1 \exp\left(-5 (1 - \min\{1, \frac{c}{50}\})^2\right)$ \\
    \bottomrule
  \end{tabular}
\caption{Detailed experiment settings for the MNIST and the CIFAR-10 datasets.}
\label{table:setting}
\end{table*}

In implementation, we used Adam\cite{kingma2014adam} optimizer for both datasets. The initial learning rate of the MNIST experiments is 0.003, and we decayed the learning rate exponentially by multiplying $\exp\left(-5 (1 - \frac{100 - c}{100})^2\right)$ after 10 epoch, where $c$ is current epoch number. And the initial learning rate of CIFAR-10 experiments is 0.001, and we multiplied the learning rate by 0.1 at 250, 350, 450, and 550 epochs. After 550 epoch, we decayed the learning rate exponentially by multiplying $\exp\left(-5 (1 - \frac{600 - c}{600 - 550})^2\right)$. For both datasets, we applied decoupled weight decay \cite{loshchilov2017decoupled} with a factor of 0.0005. Other experiment settings are shown in Table \ref{table:setting}.


\begin{table*}[ht]
  \centering
  \begin{tabular}{cc}
    \toprule
    \textbf{encoder} & \textbf{decoder} \\
    \midrule
    input: 32 $\times$ 32 $\times$ 3 Image & input: $d$-dimension latent variable\\
    Conv2D(32, 5, 2, `same'), BN, ReLU & Dense(4 $\times$ 4 $\times$ 512), BN, ReLU\\
    Conv2D(64, 4, 2, `same'), BN, ReLU & Reshape(4, 4, 512) \\
    Conv2D(128, 4, 2, `same'), BN, ReLU & Conv2DTranspose(256, 5, 2, `same'), BN, ReLU \\
    Conv2D(256, 4, 1, `same'), BN, ReLU & Conv2DTranspose(128, 5, 2, `same'), BN, ReLU \\
    Flatten & Conv2DTranspose(64, 5, 2, `same'), BN, ReLU \\
    Dense(1024), BN, Linear & Conv2DTranspose(32, 5, 1, `same'), BN, ReLU \\
    mean: K $\times$ Dense($d$), Linear & Conv2D(3, 4, 1, `same'), tanh \\
    log-variance: K $\times$ Dense($d$), Linear & - \\
    \bottomrule
  \end{tabular}
\caption{Model descriptions of the encoder and decoder. Here, $K$ denotes the number of classes, $d$ is the dimension of the latent variable, and BN is Batch-Normalization layer \cite{ioffe2015batch}. Conv2D and Conv2DTranspose hyper-parameters are filters, kernel size, strides, and padding.}
\label{table:cifar10_network}
\end{table*}

\begin{table*}[ht]
  \centering
  \begin{tabular}{c}
    \toprule
    \textbf{classifier} \\
    \midrule
    input: 32 $\times$ 32 $\times$ 3 Image \\
    Conv2D(128, 3, 1, `same'), BN, LeakyReLU($\alpha$=0.1) \\
    Conv2D(128, 3, 1, `same'), BN, LeakyReLU($\alpha$=0.1) \\
    Conv2D(128, 3, 1, `same'), BN, LeakyReLU($\alpha$=0.1) \\
    MaxPool2D(pool size=(2, 2), strides=2) \\
    SpatialDropout2D(rate=0.1) \\
    Conv2D(256, 3, 1, `same'), BN, LeakyReLU($\alpha$=0.1) \\
    Conv2D(256, 3, 1, `same'), BN, LeakyReLU($\alpha$=0.1) \\
    Conv2D(256, 3, 1, `same'), BN, LeakyReLU($\alpha$=0.1) \\
    MaxPool2D(pool size=(2, 2), strides=2) \\
    SpatialDropout2D(rate=0.1) \\
    Conv2D(512, 3, 1, `same'), BN, LeakyReLU($\alpha$=0.1) \\
    Conv2D(256, 3, 1, `same'), BN, LeakyReLU($\alpha$=0.1) \\
    Conv2D(128, 3, 1, `same'), BN, LeakyReLU($\alpha$=0.1) \\
    GlobalAveragePooling2D \\
    Dense(K), softmax \\
    \bottomrule
  \end{tabular}
\caption{Network structures of the classifier. Here, $K$ denotes the number of classes, and BN is Batch-Normalization layer \cite{ioffe2015batch}. Conv2D hyper-parameters are filters, kernel size, strides, and padding.}
\label{table:cifar10_network_classifier}
\end{table*}

\end{document}